\documentclass[12pt, a4paper]{article}

\usepackage[a4paper]{geometry}
\usepackage{pgfplots}
\pgfplotsset{
    x tick style={color=black},
    y tick style={color=black}
}

\usepackage{url}
\usepackage{listings}
\usepackage{amssymb}
\usepackage{graphicx}
\usepackage{algcompatible}
\usepackage{algorithm}
\usepackage{url}
\usepackage{rotating}
\usepackage{booktabs}
\usepackage{subfig}
\usepackage{amsmath}
\usepackage{bbm}
\usepackage{soul} 

\renewcommand{\labelenumi}{(\alph{enumi})}
\renewcommand\theenumi\labelenumi

\usepackage[english]{babel}
\usepackage{lscape}

\usepackage[utf8]{inputenc}
\usepackage{xspace}
\usepackage{amsmath,amsthm,amssymb,mathtools}
\usepackage{lmodern}

\usepackage[algo2e,ruled,vlined,linesnumbered]{algorithm2e}

\usepackage{xcolor}
\usepackage{tikz}
\usepackage{graphicx}
\usepackage{booktabs} 
\usepackage{xspace}
\usepackage[algo2e,ruled,vlined,linesnumbered]{algorithm2e}
\usepackage{pdfpages}
\usepackage{multirow}
\usepackage{siunitx}
\usepackage[square, comma, sort&compress, numbers]{natbib}

\usepackage{subfig}
\usepackage{algcompatible}
\usepackage{algorithm}
\usepackage{arydshln} 
\usepackage{ragged2e}

\renewcommand{\labelenumi}{\theenumi}
\renewcommand{\theenumi}{(\roman{enumi})}

\allowdisplaybreaks[4]
\newtheorem{theorem}{Theorem}

\newcommand{\oea}{\mbox{$(1 + 1)$~EA}\xspace}

\newcommand{\onemax}{\textsc{OneMax}\xspace}
\newcommand{\LO}{\textsc{Leading\-Ones}\xspace}
\newcommand{\leadingones}{\LO}
\newcommand{\needle}{\textsc{Needle}\xspace}

\newcommand{\binval}{\textsc{BinVal}\xspace}

\newcommand{\jump}{\textsc{Jump}\xspace}
\newcommand{\DLB}{\textsc{DeceptiveLeadingBlocks}\xspace}

\DeclareMathOperator{\poly}{poly}

\newcommand{\R}{\ensuremath{\mathbb{R}}}

\newcommand{\Z}{\ensuremath{\mathbb{Z}}}

\newcommand{\calA}{\ensuremath{\mathcal{A}}}

\newcommand{\calL}{\ensuremath{\mathcal{L}}}

\newcommand{\eps}{\varepsilon}

\usepackage{hyperref}  

\date{}
\begin{document}
{\sloppy
\title{From Understanding Genetic Drift to a Smart-Restart Mechanism for 
Estimation-of-Distribution Algorithms\thanks{A preliminary version~\cite{DoerrZ20gecco}, prepared while the first author was with Southern University of Science and Technology, was published in the proceedings of GECCO 2020.}}


\author{Weijie Zheng\\ School of Computer Science and Technology\\
International Research Institute for Artificial Intelligence\\  
       Harbin Institute of Technology\\ 
       Shenzhen, China
\and Benjamin Doerr\thanks{Corresponding author.}\\ Laboratoire d'Informatique (LIX)\\ \'Ecole Polytechnique, CNRS\\ Institut Polytechnique de Paris\\ Palaiseau, France}

\maketitle
\begin{abstract}
Estimation-of-distribution algorithms (EDAs) are optimization algorithms that learn a distribution from which good solutions can be sampled easily. A key parameter of most EDAs is the sample size (population size). Too small values lead to the undesired effect of genetic drift, while larger values slow down the process. 

Building on a quantitative analysis of how the population size leads to genetic drift, we design a smart-restart mechanism for EDAs. By stopping runs when the risk for genetic drift is high, it automatically runs the EDA in good parameter regimes. 

Via a mathematical runtime analysis, we prove a general performance guarantee for this smart-restart scheme. For many situations where the optimal parameter values are known, this shows that the restart scheme automatically finds these optimal values, leading to the asymptotically optimal performance.

We also conduct an extensive experimental analysis. On four classic benchmarks, the smart-restart scheme leads to a performance close to the one obtainable with optimal parameter values. We also conduct experiments with PBIL (cross-entropy algorithm) on the max-cut problem and the bipartition problem. Again, the smart-restart mechanism finds much better values for the population size than those suggested in the literature, leading to a much better performance.
%
%
\end{abstract}


\section{Introduction}\label{sec:intro}

Different from solution-oriented optimization heuristics such as local search, simulated annealing, or genetic algorithms, estimation-of-distribution algorithms (EDAs)~\citep{LarranagaL02,PelikanHL15} try to learn a probability distribution on the search space (``probabilistic model of the search space'') that allows to sample good solutions. EDAs are iterative in nature, that is, in each iteration they sample a certain number (``population size'') of solutions, evaluate their quality, and update the previous model based on these solutions and their quality. 

The population size is crucial for the optimization behavior and the performance of the EDA. Clearly, a large population size increases the cost of a single iteration. However, a small population size means that the model update relies only on a few samples and thus is heavily influenced by the random nature of the samples. The effect that model updates are influenced more by the randomness in the sampling process than by the guidance of the objective function is known as \emph{genetic drift}. 

Genetic drift can be detrimental to the performance of an EDA. As an example, let us discuss the performance of the univariate marginal distribution algorithm (UMDA)~\citep{MuhlenbeinP96} with artificial frequency margins $\{1/n, 1-1/n\}$ on the \DLB problem with problem size~$n$. Lehre and Nguyen~\cite[Theorem 4.9]{LehreN19foga} have shown that if the population size is small, more precisely, $\lambda=\Omega(\log n) \cap o(n)$, and the selective pressure is standard ($\mu/\lambda \ge 14/1000$), then the expected runtime is at least exponential in~$\lambda$. In contrast, if the population size is large enough, that is, $\lambda =\Omega(n\log n)$ and again $\mu=\Theta(\lambda)$, then with high probability the UMDA finds the optimum in $\lambda(n/2+2e\ln n)$ function evaluations~\cite[Theorem~5]{DoerrK21ecj}. This runtime bound is roughly proportional to the population size $\lambda$, indicating (no lower bounds were shown in~\citep{DoerrK21ecj}) that the optimal population size is just above the regime leading to the detrimental behavior observed in~\citep{LehreN19foga}, but that further increases of the population size are again costly. 

The essential reason for this performance pattern, quantified precisely by Doerr and Zheng~\cite{DoerrZ20tec}, but known already since the ground-breaking works of Shapiro~\cite{Shapiro02,Shapiro05,Shapiro06}, is that small population sizes can lead to strong genetic drift, that is, the random fluctuations of the sampling frequencies caused by the randomness in the sampling of search points eventually move some sampling frequencies towards a boundary of the frequency range that is not justified by the fitness. 

We refer to the recent survey of Krejca and Witt~\cite{KrejcaW20bookchapter} for a detailed discussion of the known runtime results for EDAs and only give a high-level summary here. For most of the results presented there, a minimum population size is necessary and then the runtime is roughly proportional to the population size. This suggests again that for small population sizes, no good performance guarantees could be proven (because of the genetic drift effect), whereas from a certain population size on this effect disappears and the runtime becomes roughly proportional to the population size (stemming from the fact that the cost of one iteration is proportional to the population size). In the presence of such a runtime behavior, naturally, choosing the appropriate population size is a key challenge in the effective usage of EDAs. 

We note that genetic drift does not in absolutely all cases lead to a bad performance. For example, independently in~\citep{DangLN19,Witt19} it was shown that the UMDA optimizes \onemax in time $O(n \log n)$ also for logarithmic population sizes, which are clearly in the regime with strong genetic drift. In~\citep{DangLN19}, it was also shown that the runtime of the UMDA on \leadingones is at most quadratic when $\lambda = \Omega(\log n) \cap O(n / \log n)$, which is again in the regime with strong genetic drift. So these results show that a good performance is also possible in the presence of strong genetic drift. We note, however, that the same runtimes of $O(n \log n)$ and $O(n^2)$ can also be obtained in the regime with low genetic drift, see again~\citep{DangLN19,Witt19}, and we note further that no example is known where an EDA has an asymptotically better performance in the strong genetic drift regime than outside of it. We finally note that the very careful experimental analysis in~\citep{LenglerSW21} shows a good runtime of the cGA on \onemax both in a range with strong genetic drift and in a range with low genetic drift (but not in between), but the runtimes observed in the former are still higher than in the latter. The experiments in~\citep{DangLN19} consider only two values for the population size $\mu$, namely $\sqrt n$ with strong genetic drift and $\sqrt n \log n$ with low genetic drift, of the UMDA optimizing \onemax and \binval. For both values good runtimes are observed, slightly better ones for the value in the genetic drift regime. However, since only two population sizes are implemented, it is not clear if really the strong genetic drift regime leads to better runtimes or if the best runtimes are observed for a value in the regime with low genetic drift, but different from the one regarded in~\citep{DangLN19}. In the light of all these results, and also our experimental results in Section~\ref{sec:exper}, it appears very justified to generally prefer running EDAs in the regime with low genetic drift.

Given the observation that genetic drift often leads to unfavorable results, there have been attempts to define EDAs that are not prone to genetic drift~\citep{Shapiro02nips,BrankeLS07,FriedrichKK16,DoerrK20tec}. While these led to some promising results, due to their restricted evaluation (only on \onemax, \binval, \leadingones, \needle, and NK-landscape) and in the light of the negative result~\cite[Theorem~4]{DoerrK20tec}, in this work we prefer to discuss how to set the parameters for established EDAs in a way that they do not suffer from genetic drift.

Setting the parameters of an optimization heuristic right is a known challenge. The most direct way is to try to understand how the parameter influences the performance of the algorithm on a given problem and then set the parameter accordingly. When done via experimental means, this approach can be time-consuming, and usually only gives information for a particular problem of a particular size. Mathematical approaches can determine optimal parameter values over larger classes of instances and problems sizes (see, e.g.~\citep{Witt13} for such a result), but they require a deep expertise and usually can only be applied to simple benchmark problems.

An easier way to approach the parameter tuning problem is to find ways to automatically set the (or some) parameters. This can be done on-the-fly, that is, the algorithm tries to learn what are good parameter values and adjusts the parameters accordingly while running, or via separate runs of the algorithm with different parameter values. From our understanding of genetic drift, we do not see how an on-the-fly parameter choice of the population size can be successful for an EDA. On the one hand, it is difficult to see during the run of the algorithm whether a model update is justified by the fitness or rather caused by unlucky samplings of the individuals. On the other hand, once the EDA has suffered from genetic drift, it is not clear how to repair the probabilistic model. 

For this reason, we shall concentrate on approaches that use several runs of the EDA with different parameter values, and this in a way that the algorithm user does not need to take care of this parameter in any way.
For EDAs having a single parameter such as the compact genetic algorithm, this will result in a parameter-less EDA.\footnote{Not surprisingly, many mechanisms to remove parameters have themselves some parameters. The name \emph{parameter-less} might still be justified when these hyperparameters have a less critical influence on the performance of the algorithm.} 

Harik and Lobo~\cite{HarikL99} proposed two strategies to remove the population size of crossover-based genetic algorithms. One basic strategy is doubling the population size and restarting when all individuals' genotypes have become identical. The drawback of this strategy is the long time it takes to fulfill this termination criterion after genetic drift has become detrimental. Harik and Lobo proposed a second strategy in which multiple populations with different sizes run simultaneously, smaller population sizes may use more function evaluations, but are removed once their fitness value falls behind the one of larger populations. Their experimental results showed that their genetic algorithm with this second strategy only had a small performance loss over the same genetic algorithm with optimal parameter settings. Many extensions of this strategy and applications to other optimization algorithms (including EDAs) have followed. These gave rise to the extended compact genetic algorithm~\citep{LimaL04}, the hierarchical Bayesian optimization algorithm~\citep{PelikanL04}, and many other algorithms. 

In the IPOP-CMA-ES, Auger and Hansen~\cite{AugerH05} use another strategy to remove the population size as a parameter to be set by the algorithm user. They restart the kernel algorithm, the $(\mu_W,\lambda)$-CMA-ES, with twice the population size (and the other parameters unchanged) once one of five predefined criteria is reached. Four of these criteria build the covariance matrix or evolution paths and thus are specific to the CMA-ES. The other criterion only depends on the objective function and thus can be used also with other algorithms (this is what we shall do in Section~\ref{ssec:other}). This criterion triggers a restart if among the last $10+\lceil 30n/\lambda \rceil$ (where $n$ is the problem size and $\lambda$ is the population size) generations, the range of the best objective function values is zero, or the range of these values together with all function values of the current generation is below a predefined threshold.

Goldman and Punch~\cite{GoldmanP14} proposed the parameter-less population pyramid, called P3, to iteratively construct a collection of populations. In P3, the population in the pyramid expands iteratively by first adding a currently not existing solution obtained by some local search strategy into the lowest population, and then utilizing some model-building methods to expand the population in all hierarchies of the pyramid. Since initially no population exists in the pyramid, this algorithm frees the practitioner from specifying a population size. 

In~\cite[Section 2.4]{Doerr21cgajump} a strategy was proposed that builds on parallel runs of EDAs with exponentially growing population sizes. With a suitable strategy to assign computational resources, this strategy needs no criterion when to abort a run and still leads to a runtime which is only by a logarithmic factor above the runtime stemming from the optimal (problem- and algorithm-specific) population size of the EDA. 

\textbf{Our contribution:}
The above parameter-less strategies can be used to automatically find good population sizes, but they are all not specific to the problem of preventing genetic drift. In this work, we aim at profiting from our understanding of genetic drift, in particular, from the recent mathematical analysis~\citep{DoerrZ20tec} which quantifies when genetic drift can arise. In very simple words, taking the compact genetic algorithm (cGA), the most simple univariate EDA,
as an example, this result shows that genetic drift has a significant influence on the sampling frequency of a bit when the number of iterations exceeds $4\mu^2$~\cite[Proof of Theorem~6]{DoerrZ20tec}, where $\mu$ is the hypothetical population size of the cGA, the only parameter of this algorithm, which plays the same role as the population size in other EDAs. We can use this insight to design the following smart-restart version of the cGA. Our \emph{smart-restart mechanism}\footnote{The authors are thankful to an anonymous reviewer of~\citep{DoerrZ20gecco} for suggesting this name.} starts with running the cGA with hypothetical population size $\mu = 2$, the smallest possible value for this parameter. After $4\mu^2$ iterations, this run is aborted and a new run is started with a randomly initialized probabilistic model, however with twice the hypothetical population size. Each new run is aborted when the time limit of $4 \mu^2$ iterations is reached and a new run with twice the parameter value is started. With this procedure, the cGA always runs in a regime in which the risk for genetic drift is considered low. The doubling scheme of the parameter value not only ensures that the possibly more effective smaller values are used first, but also ensures that their influence on the total runtime is small in the case where only a large value is successful. Like any EDA, this is an anytime algorithm, so it can be stopped at any time and then the best solution seen so far is returned.

In more detail and generality, the quantitative analysis in~\citep{DoerrZ20tec} showed that for each of the three main univariate EDAs cGA, UMDA, and PBIL, both with frequency boundaries and without, there is a number $C$ such that the probability that a particular sampling frequency of the EDA running with (hypothetical) population size $\mu$ within the first $t$ function evaluations is strongly affected by genetic drift, is at most $2\exp(-C\mu^2/t)$. This number $C$ depends on the EDA and on the value of its other parameters (in a manner made precise in~\citep{DoerrZ20tec}). Hence in all cases, genetic drift affects a sampling frequency after $\Theta(\mu^2)$ function evaluations, however, the implicit constant may depend on the precise setting. For this reason, we shall formulate our general smart-restart scheme, applicable to all these EDAs, with a hyperparameter $b$, called \emph{budget factor}, such that a run with parameter value $\mu$ is aborted after $b\mu^2$ function evaluations. In our asymptotic analyses, we allow that $b$ takes sub-constant values. The main motivation for this is that when taking $b = \Theta(1 / \log n)$, we can apply a union-bound argument to show that none of the $n$ sampling frequencies is strongly affected by genetic drift. For the increase of $\mu$ in the next stage of the algorithm, for reasons of generality, we not only consider doubling $\mu$, but multiplying it by some number $U > 1$, called update factor. We note that we have just introduced two hyperparameters for a mechanism that controls one algorithm parameter. However, as we shall see in our analyses, both theoretical and experimental, both hyperparameters are not critical for the performance of the smart-restart scheme. Taking $b = 1/\ln(n)$ and $U = 2$ gave the best asymptotic runtimes in our theoretical result and uniformly gave good results in our experiments. 

As said before, our smart-restart mechanism can be combined with any of the three main univariate EDAs, and in fact, with any heuristic \calA~ having a parameter $\mu$ such one can speculate that the runtime from a certain (unknown) value $\tilde \mu$ on is roughly linear in $\mu$. For this general setting, we prove the following \emph{mathematical runtime guarantee}. We assume that there are numbers $\tilde \mu$ and $T$ such that \calA~ with all parameter value $\mu \ge \tilde{\mu}$ solves the given problem in time $\mu T$ with probability $p > 1 - \frac 1 {U^2}$. Such a runtime behavior is very often observed in EDAs, see, e.g.,~\citep{KrejcaW20bookchapter}. We prove that under this assumption, our smart-restart mechanism with update factor $U$ and budget factor $b$ solves the problem in expected time 
\[\left(\frac{U^2}{U^2-1}+\frac{(1-p)U^2}{1-(1-p)U^2}\right)\max\left\{b\tilde{\mu}^2,\frac{T^2}{b}\right\}+\frac{pU}{1-(1-p)U}\tilde{\mu}T,\] which is $O(\max\{b \tilde \mu^2, \frac{T^2}{b}, \tilde \mu T\})$ when treating $U$ and $p$ as constants.

When combining this result with several known runtime guarantees for classic EDAs (we refer to Section~\ref{ssec:specprobs} for the details), we easily derive that the smart-restart scheme with $b = \Theta(1 / \log n)$ in these situations has the same asymptotic runtime as the original EDA with optimal (problem-specific and often non-trivial to find) value for $\mu$. This in particular holds for the analysis of the cGA on noisy \onemax functions, where the optimal value for $\mu$ depends also on the intensity of the noise, hence the parameter tuning problem is further complicated by the fact that usually the noise intensity is not known. 

We then conduct an extensive \emph{experimental analysis}. We mostly concentrate on the cGA. This algorithm has a single parameter only and this is directly controlling the strength of the model update, so it appears best to study in isolation how the model update strength and automated searches for its best value influence the performance of an EDA. 

Since a good experimental understanding of how the model update strength influences the runtime does not yet exist, and since it aids in interpreting our results for the parameter-less versions of the cGA, we also conduct experiments for the original cGA with different static parameter values. We thus ran the original cGA as well as the (parameter-less) parallel-run cGA from~\citep{Doerr21cgajump} and our (parameter-less) smart-restart cGA on the benchmarks \onemax, \LO, \jump, and \DLB, both in the absence of noise and with additive centered Gaussian posterior noise as considered in~\citep{FriedrichKKS17}. For the original cGA and the noiseless scenario, this analysis confirms, for the first time experimentally for most of these benchmarks, that small population sizes can be detrimental and that from a certain population size on, a roughly linear increase of the runtime can be observed. It also confirms experimentally the insight of the (asymptotic) theory result~\citep{Doerr21cgajump} that, with the right population size, the cGA can be very efficient on \jump functions. For example, we measure a median runtime of $4 \cdot 10^6$ on the \jump function with $n=50$ and $k=10$, parameters for which, e.g., the classic \oea would take more than $10^{17}$ iterations (the \oea is the most basic evolutionary algorithm, using a population size of one, creating one offspring from the single parent via standard bit-wise mutation, and replacing the parent by the offspring if the offspring is at least as good as the parent). In the noisy settings, we observe that the population sizes suggested (for \onemax) by the theoretical analysis~\citep{FriedrichKKS17} are much higher (roughly by a factor of $1{,}000$) than what is really necessary, leading to runtime increases of similar orders of magnitudes.

The parameter-less versions of the cGA, namely the parallel-run version and the smart-restart version with budget factors $b=16$ and $b = 1/\ln(n)$, generally perform very well. Their runtimes are, naturally, larger than the runtimes observed for the optimal static parameter value (which depends heavily on the problem and the noise level), but they clearly avoid the often catastrophic performances in the strong genetic-drift regime. Overall, the smart-restart cGA with cautious budget factor $b = 1/ \ln(n)$ appears best, in particular, for the two more difficult benchmarks \jump and \DLB. 

We also extend our smart-restart mechanism to a more complex EDA, population-based incremental learning (PBIL), which can be seen as a variant of the cross-entropy algorithm~\citep{CostaJK07,BoerKMR05}. Since this algorithm has three parameters, it is not immediately obvious how to use our smart-restart approach. We solve this problem by keeping the selection pressure $\eta=\mu/\lambda$ and the learning rate $\rho$ as parameters and by setting the sample size $\lambda$ via the smart-restart mechanism (again guided by~\cite[Theorem~3]{DoerrZ20tec}). We apply this algorithm to two combinatorial optimization problems from the literature~\citep{RubinsteinK04}, the max-cut problem and the bipartition problem. Our empirical results show that the smart-restart mechanism uses much better values for the population size than the hand-crafted parameters of the previous work~\citep{RubinsteinK04}, resulting in significant speed-ups. We also implemented the  two restart strategies from~\citep{HarikL99,AugerH05} discussed earlier. In our experiments, our smart-restart mechanism shows better performance on both combinatorial optimization problems.

We note that this version extends our preliminary version~\citep{DoerrZ20gecco}, and differs largely in the following ways. Both in the mathematical runtime analysis and in the experiments, we also consider the presence of additive posterior noise for the cGA. To demonstrate that our general approach is feasible also for other EDAs (\citep{DoerrZ20gecco} only considers the cGA), we add a runtime analysis for the smart-restart UMDA and evaluate a smart-restart version of  PBIL (cross-entropy algorithm) with two other restart strategies for comparison on two combinatorial problems. Finally, this version contains all mathematical proofs that had to be omitted in~\citep{DoerrZ20gecco} for reasons of space.

The remainder of this paper is structured as follows. Section~\ref{sec:pre} introduces the preliminaries including a detailed description of the related algorithms and benchmark functions.
The newly-proposed smart-restart mechanism will be stated in Section~\ref{sec:parameterlesscGA}. Section~\ref{sec:theory} shows our theoretical results. Section~\ref{sec:exper} contains our experimental analyses on classic benchmark functions. The extension of our mechanism to the PBIL and an experimental analysis of this algorithm on two combinatorial problems are presented in Section~\ref{sec:pbil}. Section~\ref{sec:con} concludes our paper.

\section{Preliminaries}
\label{sec:pre}

\subsection{Algorithms}\label{subsec:algos}
In this paper, we consider algorithms maximizing pseudo-Boolean functions $f:\{0,1\}^n \rightarrow \R$. We regard so-called anytime algorithms, that is, algorithms that can be stopped at any time, and then return the best solution seen so far. In practice, such algorithms are run with some user-specified termination criterion. In our mathematical analyses, we regard the time it takes until an optimum is generated if the algorithm is not stopped prematurely. For that reason, we do not specify a termination criterion here.

Since our smart-restart mechanism builds on an original heuristic, such as the cGA of Harik et al.~\cite{HarikLG99} or the UMDA of M{\"{u}}hlenbein and Paass~\cite{MuhlenbeinP96}, and since we will compare our smart-restart cGA with the parallel-run cGA~\citep{Doerr21cgajump}, this subsection will give a brief introduction to these algorithms. We shall, in Section~\ref{sec:pbil}, also discuss the general performance of our smart-restart PBIL (cross-entropy algorithm), and for the convenience of reading, we will introduce this algorithm in Section~\ref{sec:pbil}. 

In the following, we shall use $X^g_{i}=(X^g_{i,1},\dots,X^g_{i,n}) \in \{0,1\}^n$ to denote the $i$-th bitstring in the $g$-th iteration of the algorithm and $X^g_{i,j}$ for the $j$-th bit of $X^g_{i}$. We use $p^g=(p^g_{1},\dots,p^g_{n})$ to denote the {(univariate)} probabilistic model learned in the $g$-th iteration and $p^g_{j}$ for the $j$-th entry of $p^g$.

\subsubsection{The Compact Genetic Algorithm}\label{subsec:cga}
The compact genetic algorithm (cGA) with hypothetical population size $\mu$ samples two individuals in each generation and moves the sampling frequencies by an absolute value of $1/\mu$ towards the bit values of the better individual. Usually, and so do we, in order to avoid frequencies reaching the absorbing boundaries $0$ or $1$, the artificial margins $1/n$ and $1-1/n$ are utilized, that is, we restrict the frequency values to be in the interval $[1/n,1-1/n]$. The following Algorithm~\ref{alg:cGA} shows the details. As is common in runtime analysis, we do not specify a termination criterion. When talking about the runtime of an algorithm, we mean the first time (measured by the number of fitness evaluations) an optimum was sampled.
\begin{algorithm}[!ht]
\caption{The cGA to maximize a function $f: \{0,1\}^n \rightarrow \R$ with hypothetical population size $\mu$}
{\small
 \begin{algorithmic}[1]
 \STATE{$p^0=(\tfrac{1}{2}, \tfrac{1}{2},\dots,\tfrac{1}{2})\in [0,1]^n$}
 \FOR {$g=1,2,\dots$}
 \STATEx {$\quad\%\%$\textsl{Sample two individuals $X_1^g,X_2^g$}}
 \FOR {$i=1,2$}
 \FOR {$j=1,2,\dots,n$}
 \STATE $X_{i,j}^g \leftarrow 1$ with probability $p_{j}^{g-1}$ and $X_{i,j}^g \leftarrow 0$ with probability $1-p_{j}^{g-1}$
 \ENDFOR
 \ENDFOR
 \STATEx {$\quad\%\%$\textsl{Update of the frequency vector}}
 \IF{$f(X_1^g) \ge f(X_2^g)$}
 \STATE {$p'=p^{g-1}+\tfrac{1}{\mu}(X_1^g-X_2^g)$}
 \ELSE 
  \STATE {$p'=p^{g-1}+\tfrac{1}{\mu}(X_2^g-X_1^g)$}
  \ENDIF
 \STATE {$p^g=\min \{\max\{\tfrac{1}{n},p'\},1-\tfrac{1}{n}\}$}
 \ENDFOR
 \end{algorithmic}
 \label{alg:cGA}
}
\end{algorithm}

\subsubsection{The Univariate Marginal Distribution Algorithm}\label{subsubsec:umda}

The univariate marginal distribution algorithm (UMDA) samples $\lambda$ individuals in each generation and selects the best $\mu$ individuals to learn its probabilistic model. More precisely, the new sampling frequency $p_i$ of the $i$-th bit is taken as the ratio of ones in the $i$-th bit of the $\mu$ selected individuals. Similar to the cGA, the artificial margins $1/n$ and $1-1/n$ are utilized to avoid the frequencies reaching the boundaries $0$ or $1$. See Algorithm~\ref{alg:umda} for details.
\begin{algorithm}[!ht]
\caption{The UMDA (with sample size $\lambda$, and selection size $\mu$) to maximize a function $f: \{0,1\}^n \rightarrow \R$}
{\small
 \begin{algorithmic}[1]
 \STATE{$p^0=(\tfrac{1}{2}, \tfrac{1}{2},\dots,\tfrac{1}{2})\in [0,1]^n$}
 \FOR {$g=1,2,\dots$}
 \STATEx {$\quad\%\%$\textsl{Sample $\lambda$ individuals $X_1^g,\dots,X_{\lambda}^g$}}
 \FOR {$i=1,2,\dots,\lambda$}
 \FOR {$j=1,2,\dots,n$}
 \STATE $X_{i,j}^g \leftarrow 1$ with probability $p_{j}^{g-1}$ and $X_{i,j}^g \leftarrow 0$ with probability $1-p_{j}^{g-1}$
 \ENDFOR
 \ENDFOR
 \STATEx {$\quad\%\%$\textsl{Update of the frequency vector}}
 \STATE Let $\tilde{X}_1^g,\dots,\tilde{X}_{\mu}^g$ be the best $\mu$ individuals (tie broken randomly)
 \STATE {$p' = \frac{1}{\mu} \sum_{i=1}^{\mu}\tilde{X}_i^g$}
 \STATE {$p^g=\min \{\max\{\tfrac{1}{n}, p' \},1-\tfrac{1}{n}\}$}
 \ENDFOR
 \end{algorithmic}
 \label{alg:umda}
}
\end{algorithm}

\subsubsection{The Parallel-run cGA}
\label{subsec:prllcGA}
The parallel EDA mechanism was proposed by Doerr~\cite{Doerr21cgajump} as a side result when discussing the connection between runtime bounds that hold with high probability and the expected runtime. For the cGA, this mechanism yields the following \emph{parallel-run cGA}. In the initial round $\ell = 1$, we start process $\ell=1$ to run the cGA with population size $\mu=2^{\ell -1}$ for $1$ generation. In round $\ell = 2, 3, \dots$, all running processes $j=1,\dots, \ell-1$ run $2^{\ell-1}$ generations and then we start process $\ell$ to run the cGA with population size $\mu=2^{\ell-1}$ for $\sum_{i=0}^{\ell-1} 2^i$ generations. The algorithm terminates once any process has found the optimum. 
Algorithm~\ref{alg:prllcGA} shows the details of the parallel-run cGA.

Based on the following assumption, Doerr~\cite{Doerr21cgajump} proved that the expected runtime for this parallel-run cGA is at most $6 \tilde \mu T (\log_2(\tilde \mu T) + 3)$.

\textbf{Assumption~\citep{Doerr21cgajump}:} Consider using the cGA with population size $\mu$ to maximize a given function $f$. Assume that there are unknown $\tilde{\mu}$ and $T$ such that the cGA for all population sizes $\mu \ge \tilde{\mu}$ optimizes this function $f$ in $\mu T$ fitness evaluations with probability at least $\tfrac 34$.\\

\begin{algorithm}[!ht]
\caption{The parallel-run cGA to maximize a function $f: \{0,1\}^n \rightarrow \R$}
{\small
 \begin{algorithmic}[1]
 \STATE Process 1 runs cGA (Algorithm~\ref{alg:cGA}) with population size $\mu=1$ for $1$ generation
 \FOR {round $\ell=2,\dots$}
 \STATE Processes $1,\dots, \ell-1$ continue to run for another $2^{\ell-1}$ generations, one process after the other one
 \STATE Start process $\ell$ to run cGA (Algorithm~\ref{alg:cGA}) with population size $\mu=2^{\ell-1}$ to maximize $f$ and run it for $\sum_{i=0}^{\ell-1}2^i$ generations 
 \ENDFOR
 \end{algorithmic}
 \label{alg:prllcGA}
}
\end{algorithm}


\subsection{Benchmark Functions}\label{ssec:benchmarks}

To understand the performance of the mechanism proposed in this work, we regard four basic benchmark functions (detailed in Section~\ref{sssec:basic} below). They are all popular benchmarks in the analysis of evolutionary algorithms, so they are well-understood both from a theoretical and an experimental point of view, which helps interpreting our results. To further test the performance of the proposed approach in a more complicated setting, we also consider a noisy environment (detailed in Section~\ref{ssec:intronoise}). We shall, in Section~\ref{sec:pbil}, also conduct an experimental investigation of the PBIL (cross-entropy algorithm) on particular instances of two combinatorial optimization problems (namely those suggested in~\citep{RubinsteinK04} to show the power of this algorithm), but since these results cannot be easily compared to the other ones,  among others, because no theory exists for these instances, we describe these problems not here, but in Section~\ref{sec:pbil}.

\subsubsection{Basic Benchmark Functions}\label{sssec:basic}

We selected the four benchmark functions \onemax, \LO, \jump, and \DLB as optimization problems. 
All four problems are defined on binary representations (bit strings) and we use $n$ to denote their length, that is, all are functions $f:\{0,1\}^n \rightarrow \R$.

The \textbf{\onemax} problem is one of the easiest benchmark problems. The \onemax fitness of a bit string is simply the number of ones in the bit string. Formally, the \onemax function value of any $x=(x_1,\dots,x_n)\in\{0,1\}^n$ is defined by 
$$
f(x)= \sum_{i=1}^n x_i.
$$ 
Having the perfect fitness-distance correlation, most evolutionary algorithms find it easy to optimize \onemax. A common and often easy to prove runtime is $\Theta(n \log n)$~\citep{Muhlenbein92,GarnierKS99,JansenJW05,Witt06,RoweS14,DoerrK15,AntipovD21algo,OlivetoSW22}. For EDAs, apparently, the runtime of \onemax is more complicated. The known results for EDAs are the following. The first mathematical runtime analysis for EDAs by Droste~\cite{Droste06} together with the recent work~\citep{SudholtW19} shows that the cGA can efficiently optimize \onemax in time $\Theta(\mu\sqrt n)$ when $\mu \ge K \sqrt n \ln(n)$ for some sufficiently large constant $K$. As the proofs of this result show (and the same could be concluded from the general result~\citep{DoerrZ20tec}), in this parameter regime there is little genetic drift. Throughout the runtime, with high probability, all bit frequencies stay above $\frac 14$. For hypothetical population sizes below the $\sqrt n \log n$ threshold, the situation is less understood. However, the lower bound of $\Omega(\mu^{1/3} n)$ valid for all $\mu = O\left(\frac{\sqrt n}{\ln(n) \ln\ln(n)}\right)$ proven in~\citep{LenglerSW21} together with its proof shows that in parts of this regime the cGA suffers from genetic drift, leading to (mildly) higher runtimes. 

For another EDA, the UMDA with sampling population size $\lambda$ and selected population size $\mu$, a runtime of $\Omega(\lambda \sqrt n+n\log n)$ for $\mu=\Theta(\lambda)$ and $\lambda \in \poly(n)$ is shown in~\citep{KrejcaW20}. Upper bounds were independently proven in~\citep{Witt19,DangLN19}. For $\mu=\Theta(\lambda)$,~\citep{Witt19} proves a runtime of $O(\lambda n)$ for $\lambda=\Omega(\log n) \cap o(n)$ and a runtime of $O(\lambda \sqrt n)$ for $\lambda=\Omega(\sqrt n \log n) \cap n^{O(1)}$. Replacing the restriction $\mu=\Theta(\lambda)$ with $\lambda=\Omega(\mu)$, in~\citep{DangLN19}, an upper bound of $O(\lambda n)$ is obtained for $\lambda=\Omega(\log n) \cap O(\sqrt n)$ and one of $O(\lambda \sqrt n)$ for $\lambda=\Omega(\sqrt n \log n)$.

The \textbf{\LO} benchmark is still an easy unimodal problem, however, typically harder than \onemax. The \leadingones value of a bit string is the number of ones in it, counted from left to right, until the first zero. Formally, the \LO function value of any $x=(x_1,\dots,x_n)\in\{0,1\}^n$ is defined by 
$$
f(x)= \sum_{i=1}^n \prod_{j=1}^ix_j.
$$
How simple randomized search heuristics optimize \leadingones is extremely well understood~\citep{DrosteJW02,JansenJW05,Witt06,BottcherDN10,Sudholt13,Doerr19tcs,LissovoiOW20ecj,Sudholt21,DoerrDL21}, many EAs optimize this benchmark in time $\Theta(n^2)$. Surprisingly, no theoretical results are known on how the cGA optimizes \leadingones. However, the runtime of the UMDA with population sizes $\mu = \Theta(\lambda)$ with suitable implicit constants and $\lambda = \Omega(\log n)$ was shown to be $O(n \lambda \log(\lambda) + n^2)$~\citep{DangLN19} and, recently, $\Theta(n \lambda)$ for $\lambda = \Omega(n \log n)$~\citep{DoerrK21tcs}. We remark that~\citep{DoerrZ20tec} for this situation shows that genetic drift occurs when $\lambda = O(n)$ (with suitable implicit constants). Consequently, these results show a roughly linear influence of $\lambda$ on the runtime when $\lambda$ is (roughly) at least linear in $n$, but below this value, there is apparently no big penalty for running the EDA in the genetic drift regime. For the cGA, we will observe a different behavior, which also indicates that translating general behaviors from one EDA to another, even within the class of univariate EDAs, has to be done with caution.

The \textbf{\jump} benchmark is a class of multimodal fitness landscapes of scalable difficulty. For a difficulty parameter $k$, the fitness landscape is isomorphic to the one of \onemax except that there is a valley of low fitness of width $k$ around the optimum. More precisely, all search points in distance $1$ to $k-1$ from the optimum have a fitness lower than all other search points. 
Formally, the $\jump_k$ function (with $k \in [1..n]$) is defined by   
\begin{align*}
f(x)&=
\begin{cases}
k+\sum_{i=1}^nx_i, & \text{if~} \sum_{i=1}^nx_i \le n-k \text{~or~} x=1^n,\\
n-\sum_{i=1}^nx_i, & \text{else},
\end{cases}
\end{align*} 
for any $x=(x_1,\dots,x_n)\in\{0,1\}^n$.

Recent results~\citep{HasenohrlS18,Doerr21cgajump} show that when $\mu$ is large enough, then the cGA can optimize \jump functions quite efficiently and significantly more efficient than many classic evolutionary algorithms. We omit some details and only mention that for $k$ not too small, a runtime exponential in $k$ results from a population size $\mu$ that is also exponential in $k$. This is much better than the $\Omega(n^k)$ runtime of typical mutation-based evolutionary algorithms~\citep{DrosteJW02,DoerrLMN17,Doerr22,RajabiW22} or the $n^{O(k)}$ runtime bounds shown for several crossover-based algorithms~\citep{DangFKKLOSS16,AntipovDK22}. We note that $O(n)$ and $O(n \log n)$ runtimes have been shown in \citep{WhitleyVHM18,RoweA19}, however, these algorithms appear quite problem-specific~\citep{Witt23} and have not been regarded in other contexts so far. It was not known whether the runtime of the cGA becomes worse in the regime with genetic drift, but our experimental results now show an enormously weak performance in this regime.

The \textbf{\DLB} benchmark was introduced in~\citep{LehreN19foga}. It can be seen as a deceptive version of the \leadingones benchmark. In \DLB, the bits are partitioned into blocks of length two in a left-to-right fashion. The fitness is computed as follows. Counting from left to right, each block that consists of two ones contributes two to the fitness, until the first block is reached that does not consist of two ones. This block contributes one to the fitness if it consists of two zeros, otherwise it contributes zero. All further blocks do not contribute to the fitness. 
Formally, the \DLB (requiring $n$ is even) function value of any $x=(x_1,\dots,x_n)\in\{0,1\}^n$ is defined by 
\begin{align*}
f(x)&=
\begin{cases}
2m+1, & \text{if~} x_{[1..2m]}=1^{2m}\text{~and~} x_{2m+1}=x_{2m+2}=0,\\
2m, & \text{if~} x_{[1..2m]}=1^{2m}\text{~and~} x_{2m+1}+x_{2m+2}=1,\\
n, & \text{if~} x=1^n.
\end{cases}
\end{align*} 
The main result in~\citep{LehreN19foga} is that when $\mu = \Theta(\lambda)$ and $\lambda = o(n)$, the expected runtime of the UMDA on \DLB is $\exp(\Omega(\lambda))$. We note that when $\lambda = o(n)$, already after a quadratic runtime strong genetic drift is encountered according to~\citep{DoerrZ20tec}. When $\lambda = \Omega(n \log n)$, a runtime guarantee of at most $(1+o(1)) \frac 12 \lambda n$ holds with high probabilitiy~\citep{DoerrK21ecj}. Hence for this function and the UMDA as an optimizer, the choice of the population size is again very important. This was the reason for including this function into our set of test problems and the results indicate that indeed the cGA shows a behavior similar to what the mathematical results showed for the UMDA. 

We note that \citep{LehreN19foga} also shows an expected runtime of $O(n\lambda\log \lambda + n^3)$ when $\mu = \Omega(\log n)$ and $\lambda=\Omega(\mu^2)$, which is an unusually high selection pressure. Other runtime results on the \DLB function include several $O(n^3)$ runtime guarantees for classic EAs~\citep{LehreN19foga} as well as a $\Theta(n^2)$ runtime for the Metropolis algorithm and an $O(n \log n)$ runtime guarantee for the significance-based cGA~\citep{WangZD21}. Till now, there is no theoretical runtime analysis for the cGA.

\subsubsection{Additive Centered Gaussian Posterior Noise}\label{ssec:intronoise}

In practical applications, one often encounters various forms of uncertainty. One of these is a noisy access to the objective function. Friedrich et al.~\cite{FriedrichKKS17} analyzed how the cGA optimizes the \onemax problem under additive centered Gaussian posterior noise. They proved that for all noise intensities (variances $\sigma^2$ of the Gaussian distribution), there is a population size $\mu = \mu(\sigma^2)$ which depends only polynomially on $\sigma^2$ (that is, $\mu(\sigma^2)$ is a polynomial in $\sigma^2$) so that the cGA with this population size efficiently solves the \onemax problem. This was called \emph{graceful scaling}. They also provided a restart scheme that obtains this performance without knowledge of the noise intensity (however, it requires to know the polynomial $\mu(\sigma^2)$). Hence these results show that the cGA can deal well with the type of noise regarded, and much better than many classic evolutionary algorithms (see the lower bounds in~\citep{GiessenK16,FriedrichKKS17}), but this still needs an action by the algorithm user, namely an appropriate choice of the population size $\mu$.

As we shall show in this work, our restart scheme is also able to optimize noisy versions of \onemax and many other problems, but without knowing the polynomial $\mu(\sigma^2)$ and using significantly more efficient values for the population size. For \onemax, we prove rigorously that we obtain essentially the performance of the original cGA with the best choice of the population size (Theorem~\ref{thm:noisyom}), where we profit from the fact that the runtime analysis of~\citep{FriedrichKKS17} shows that the cGA also for noisy \onemax functions essentially satisfies our main assumption that from a certain population size on, the runtime of the cGA is at most proportional to the population size. 

We conduct experiments for various benchmark functions in this noise model. They indicate that also for problems different from \onemax, the graceful scaling property holds. However, they also show that much smaller population sizes suffice to cope with the noise. Consequently, our smart-restart cGA (as well as the parallel-run cGA from~\citep{Doerr21cgajump}) optimizes \onemax much faster than the algorithms proposed in~\citep{FriedrichKKS17}. This is natural since the parameter-less approaches also try smaller (more efficient in case of success) population sizes, whereas the approaches in~\citep{FriedrichKKS17} use a population size large enough that one can prove via mathematical means that they will be successful with high probability. 

We now make precise the additive centered Gaussian noise model. We take the common assumption that whenever the noisy fitness of a search point is regarded in a run of the algorithm, its noisy fitness is computed anew, that is, with newly sampled noise. This avoids that a single exceptional noise event misguides the algorithm for the remaining run. A comparison of the results in~\citep{SudholtT12} (without independent reevaluations) and~\citep{DoerrHK12ants} (with reevaluations) shows how detrimental sticking to previous evaluations can be. We regard posterior noise, that is, the noisy fitness value is obtained from a perturbation of the original fitness value (independent of the argument) as opposed to prior noise, where the algorithm works with the fitness of a perturbed search point. We regard additive perturbations, hence the perceived fitness of a search point $x$ is $f(x)+D$, where $f$ is the original fitness function and $D$ is an independent sample from a distribution describing the noise. Since we consider centered Gaussian noise, we always have $D \sim \mathcal{N}(0,\sigma^2)$, where $\mathcal{N}(0,\sigma^2)$ denotes the Gaussian distribution with expectation zero and variance $\sigma^2 \ge 0$. Obviously, the classic noise-free optimization scenario is subsumed by the special case $\sigma^2=0$. 

\section{The Smart-Restart Mechanism}
\label{sec:parameterlesscGA}

In this section, we introduce our \emph{smart-restart mechanism}. It can be applied to any randomized search heuristic $\calA$ having an integral parameter $\mu$ and it makes sense when we can assume that the algorithm has a good performance from a certain value for $\mu$ on, a situation often encountered in EDAs. In contrast to the parallel-run mechanism proposed in~\citep{Doerr21cgajump}, which applies to the same scenario, the smart-restart mechanism does not run processes in parallel, which is an advantage from the implementation point of view. The main advantage we aim for is that by predicting when runs with a certain parameter value become hopeless, we can abort these runs and save runtime.

As in~\citep{Doerr21cgajump}, in this exposition we let ourselves be guided by the cGA as base algorithm~$\calA$ as it is maybe the simplest algorithm in which a parameter behavior as sketched above is encountered. To decide when to abort a run, we use the first tight quantification of the genetic drift effect of the EDAs by Doerr and Zheng~\cite{DoerrZ20tec}. In detail, they proved that in a run of the cGA with hypothetical population size $\mu$ a frequency of a neutral bit will reach the boundaries of the frequency range in an expected number of at most $4\mu^2$ generations (equivalent to $8\mu^2$ fitness evaluations), which is asymptotically tight. By Markov's inequality the probability that a boundary is reached in $b\mu^2, b>{8}$, fitness evaluations, is at least $1-{8}/b$. 

This finding inspires the following restart scheme (for any randomized search heuristic $\calA$ with a parameter $\mu$), also described in Algorithm~\ref{alg:nonpcGA}. We repeat running algorithm $\calA$ with increasing values for $\mu$, each time until we decide to abort such a run. For the $\ell$-th run, $\ell = 1, 2, \dots$, we use the parameter value $\mu_\ell = 2 U^{\ell-1}$, that is, we start with a small value $\mu=2$ and increase $\mu$ by a factor of $U > 1$, called \emph{update factor}, from one run to the next. We abort the $\ell$-th run after $B_\ell = b \mu_\ell^2$ fitness evaluations, where $b$ is the second parameter of the restart scheme, called \emph{budget factor}. As before, we do not specify a termination criterion since for our analysis we just count the number of fitness evaluations until a desired solution is found. 

We shall discuss the settings of the parameters $U$ and $b$ later in more detail. As a motivating example, we note that for the cGA by taking $b = 16$, the above Markov bound argument shows that in each such run, each bit has a probability of at most $\frac 12$ to be subject to strong genetic drift. Hence this restart scheme manages to run the cGA in a way that genetic drift does not affect too many bits. We shall later see that a budget factor $b = O(1 / \log n)$ can even ensure that none of the bits is subject to strong genetic drift.

\begin{algorithm}[!ht]
\caption{The smart-restart mechanism with update factor $U$ and budget factor $b$ applied to an algorithm $\calA$ with parameter $\mu$ for the maximization of a function $f: \{0,1\}^n \rightarrow \R$. }
{\small
 \begin{algorithmic}[1]
 \FOR { $\ell=1,2,\dots$}
 \STATE Run $\calA$ with parameter value $\mu_{\ell}=2U^{\ell-1}$ for $B_{\ell}=b\mu_{\ell}^2$ fitness evaluations on the maximization problem~$f$
 \ENDFOR
 \end{algorithmic}
 \label{alg:nonpcGA}
}
\end{algorithm}

\section{Theoretical Analyses}
\label{sec:theory}

In this section, we prove mathematical runtime guarantees for our smart-restart mechanism.

\subsection{A General Performance Guarantee}

We follow the general approach of~\cite[Section~2.4]{Doerr21cgajump} of assuming that the runtime increases linearly with the population size from a given minimum size $\tilde \mu$ on. We do not restrict this property for the cGA with population size $\mu$, but for the general EDA (or any randomized search heuristic) with a parameter $\mu$.

\textbf{Assumption (L):} Let $p \in (0,1]$. Consider using an EDA (or a randomized search heuristic) with parameter~$\mu$ to maximize a given function $f$. Assume that there are unknown $\tilde{\mu}$ and $T$ such that the EDA (or randomized search heuristic) for all parameter values $\mu \ge \tilde{\mu}$ optimizes $f$ within $\mu T$ fitness evaluations with probability at least~$p$.

This Assumption~(L) is identical to the assumption taken in~\citep{Doerr21cgajump} except that there $p$ was required to be at least $3/4$, whereas we allow a general positive~$p$ (but note that we will require $p > 1 - \frac{1}{U^2}$, hence a small $p$ limits the choice of $U$). Since most existing runtime analyses give bounds with success probability $1-o(1)$, this difference is, of course, not very important. We note that the proof of the result in~\citep{Doerr21cgajump} requires $p$ to be at least $3/4$, but we also note that an elementary probability amplification argument (via independent restarts) allows to increase the success probability of a given algorithm, so that the result of~\citep{Doerr21cgajump} becomes applicable to this modified algorithm.

Under this Assumption~(L), we obtain the following result. We note that it is non-asymptotic, which later allows to easily obtain asymptotic results also for non-constant parameters. We note that when assuming $p$ and $U$ to be constants (which is very natural), then the bound becomes $O(\max\{b\tilde{\mu}^2,{T^2}/{b},\tilde{\mu}T\})$.

\begin{theorem}
Let  $U >1$ and $b > 0$. Consider using the smart-restart mechanism on an EDA (or a randomized search heuristic)
with update factor $U$ and budget $B_{\ell}=b\mu_{\ell}^2, \ell=1,2,\dots$, optimizing a function $f$ satisfying Assumption~(L) with $p \in (1-\frac{1}{U^2},1]$. Then the expected time until the optimum of $f$ is generated is at most 
\[
\left(\frac{U^2}{U^2-1}+\frac{(1-p)U^2}{1-(1-p)U^2}\right)\max\left\{b\tilde{\mu}^2,\frac{T^2}{b}\right\}+\frac{pU}{1-(1-p)U}\max\left\{\tilde{\mu}T,\frac{T^2}{b}\right\}
\]
fitness evaluations.
\label{thm:nonpcGAwAssume}
\end{theorem}

\begin{proof}
Let $\ell' = \min\{\ell \mid 2U^{\ell-1} \ge \tilde{\mu}, B_{\ell} \ge 2U^{\ell-1} T\} $. With $B_{\ell}=b\mu_{\ell}^2$ and $\mu_{\ell}=2U^{\ell-1}$, we have $\ell' = \min\{\ell \mid 2U^{\ell-1} \ge \tilde{\mu}, b(2U^{\ell-1})^2 \ge 2U^{\ell-1} T\}= \min\{\ell \mid 2U^{\ell-1} \ge \tilde{\mu}, 2U^{\ell-1} \ge T/b\}$, that is, $\ell' = \min\{\ell \mid 2U^{\ell-1} \ge \max\{\tilde{\mu},T/b\}\}$. Then it is not difficult to see that $2U^{\ell'-1} \le U\max\{\tilde{\mu},T/b\}$ and that for any $\ell \ge \ell'$, the population size $\mu_{\ell} := 2U^{\ell-1}$ satisfies $\mu_{\ell} \ge \tilde{\mu}$ and $B_{\ell} \ge \mu_{\ell} T$. Hence, according to the assumption, we know the EDA (randomized search heuristic) with such a $\mu_{\ell}$ optimizes $f$ with probability at least $p$ in time~$\mu_{\ell} T$. We pessimistically assume that the optimum is not reached before the parameter value increases to $\mu_{\ell'}$. Now the expected time when the smart-restart EDA (randomized search heuristic) finds the optimum of $f$ is at most
\begin{align*}
\sum_{i=1}^{\ell'-1}B_i&{}+p\cdot 2U^{\ell'-1}T + \sum_{i=1}^{\infty}(1-p)^ip\left(\sum_{j=0}^{i-1} B_{\ell'+j} + 2U^{\ell'+i-1}T\right)\\
= &{} \sum_{i=1}^{\ell'-1}B_i+p\cdot 2U^{\ell'-1}T + \sum_{i=1}^{\infty}(1-p)^ip\sum_{j=0}^{i-1} B_{\ell'+j} + 2pU^{\ell'-1}T\sum_{i=1}^{\infty}(1-p)^iU^{i}\\
= &{} \sum_{i=1}^{\ell'-1}B_i+2pU^{\ell'-1}T + \sum_{j=0}^{\infty}B_{\ell'+j}p \sum_{i=j+1}^{\infty} (1-p)^i + 2pU^{\ell'-1}T\frac{(1-p)U}{1-(1-p)U}\\
= &{} \sum_{i=1}^{\ell'-1}B_i+\sum_{j=0}^{\infty}B_{\ell'+j}p \frac{(1-p)^{j+1}}{1-(1-p)}+\frac{2pU^{\ell'-1}T}{1-(1-p)U}\\
= &{} \sum_{i=1}^{\ell'-1}B_i+\sum_{j=0}^{\infty}(1-p)^{j+1} B_{\ell'+j}+\frac{2pU^{\ell'-1}T}{1-(1-p)U},
\end{align*}
where the second equality uses $(1-p)U \in [0,1)$ from $p \in (1-\frac{1}{U^2},1]$. With $B_{\ell}=b\mu_{\ell}^2=b(2U^{\ell-1})^2=4bU^{2\ell-2}$, we further compute
\begin{align*}
\sum_{i=1}^{\ell'-1}B_i{}&{}+\sum_{j=0}^{\infty}(1-p)^{j+1} B_{\ell'+j}+\frac{2pU^{\ell'-1}T}{1-(1-p)U}\\
= {}&{} \sum_{i=1}^{\ell'-1}4bU^{2i-2}+\sum_{j=0}^{\infty}(1-p)^{j+1} 4bU^{2\ell'+2j-2}+\frac{2pU^{\ell'-1}T}{1-(1-p)U}\\
= {}&{} \frac{4b(U^{2\ell'-2}-1)}{U^2-1} + \frac{4b(1-p)U^{2\ell'-2}}{1-(1-p)U^2}+\frac{2pU^{\ell'-1}T}{1-(1-p)U}\\
\le {}&{} \frac{bU^2\max\{\tilde{\mu}^2,T^2/b^2\}}{U^2-1} + \frac{b(1-p)U^2\max\{\tilde{\mu}^2,T^2/b^2\}}{1-(1-p)U^2}+\frac{pU\max\{\tilde{\mu}, T/b\}T}{1-(1-p)U}\\
= {}&{} \left(\frac{U^2}{U^2-1}+\frac{(1-p)U^2}{1-(1-p)U^2}\right)\max\left\{b\tilde{\mu}^2,\frac{T^2}{b}\right\}+\frac{pU}{1-(1-p)U}\max\left\{\tilde{\mu}T,\frac{T^2}{b}\right\},
\end{align*}
where the second equality uses $(1-p)U^2 \in [0,1)$ from $p \in (1-\frac{1}{U^2},1]$ and the first inequality uses $2U^{\ell'-1} \le U\max\{\tilde{\mu},T/b\}$.\footnote{We thank Jesse Geneson (geneson@gmail.com) for pointing out the typo in the third item on the fourth line of the above calculation, where we previously wrote $\frac{pU\tilde{\mu}T}{1-(1-p)U}$.}
\end{proof}

For comparison, we recall that the complexity of the parallel-run cGA from~\citep{Doerr21cgajump}. 
\begin{theorem}[{\cite[Theorem~2]{Doerr21cgajump}}]
The expected number of fitness evaluations for the parallel-run cGA optimizing a function $f$ satisfying Assumption~(L) with $p \ge 3/4$ is ${O\left(\tilde{\mu}T\log(\tilde{\mu}T)\right)}$.
\label{thm:parallel}
\end{theorem}

Since the choice $b = \Theta(T / \tilde \mu)$ gives an asymptotic runtime of $O(\tilde \mu T)$ for the smart-restart cGA, we see that with the right choice of the parameters the smart-restart cGA can outperform the parallel-run cGA slightly. This shows that it indeed gains from its ability to abort unprofitable runs.

Our main motivation for regarding Assumption (L) was that this runtime behavior is often observed both in theoretical results (see, e.g., the survey~\citep{KrejcaW20bookchapter}) and in experiments (see Section~\ref{sec:exper}). Unfortunately, some theoretical results were only proven under the additional assumption that $\mu$ is polynomially bounded in~$n$, that is, that $\mu = O(n^C)$ for some, possibly large, constant $C$. For most of these results, we are convinced that the restriction on $\mu$ is not necessary, but was only taken for convenience and in the light that super-polynomial values for $\mu$ would imply not very interesting super-polynomial runtimes. To extend such results to our smart-restart cGA in a formally correct manner, we now prove a version of Theorem~\ref{thm:nonpcGAwAssume} applying to such settings. More precisely, we regard the following assumption. Similar to Assumption (L), we do not restrict this property to the cGA with the population size as parameter.

\textbf{Assumption (L'):} Let $p \in (0,1]$. Consider using an EDA (or a randomized search heuristic) with parameter~$\mu$ to maximize a given function $f$. Assume that there are unknown $\tilde{\mu}$, $\mu^+$,  and $T$ such that the EDA (or randomized search heuristic) for all parameter values $\tilde \mu \le \mu \le \mu^+$ optimizes $f$ within $\mu T$ fitness evaluations with probability at least~$p$.

We prove the following result.

\begin{theorem}
Let  $U >1$ and $b > 0$. Consider using the smart-restart mechanism on an EDA (or a randomized search heuristic) with update factor $U$ and budget $B_{\ell}=b\mu_{\ell}^2, \ell=1,2,\dots$, optimizing a function $f$ satisfying Assumption~(L') with $p \in (1-\frac{1}{U^2},1)$. Let $\ell' := \min\{\ell \mid 2U^{\ell-1} \ge \tilde{\mu}, B_{\ell} \ge 2U^{\ell-1} T\}$ and $\calL := \{\ell \in \Z \mid \ell \ge \ell', 2 U^{\ell-1} \le \mu^+\}$. Then, apart from when an exceptional event of probability at most $(1-p)^{|\calL|}$ holds, the expected time until the optimum of $f$ is generated is at most 
\[
\left(\frac{U^2}{U^2-1}+\frac{(1-p)U^2}{1-(1-p)U^2}\right)\max\left\{b\tilde{\mu}^2,\frac{T^2}{b}\right\}+\frac{pU}{1-(1-p)U}\max\left\{\tilde{\mu}T,\frac{T^2}{b}\right\}
\]
fitness evaluations.
\label{thm:nonpcGAwAssume2}
\end{theorem}

\begin{proof}
Let $A$ be the event that none of the runs of the EDA (or randomized search heuristic) with parameter $\mu_\ell = 2 U^{\ell-1}$ at most $\mu^+$ finds the optimum of $f$. As in the proof of Theorem~\ref{thm:nonpcGAwAssume}, each of the runs using parameter $\mu_\ell$, $\ell \in \calL$, with probability at least $p$ finds the optimum. Hence the event $A$ occurs with probability at most $(1-p)^{|\calL|}$. 
	
	Let us condition on the event $\neg A$. Under this event, the smart-restart EDA (or randomized search heuristic) surely finds the optimum before the parameter value is increased beyond $\mu^+$. Note that when running the EDA (or randomized search heuristic) with a parameter value of at most $\mu^+$, the Assumptions (L) and (L') are identical. For that reason, analogous to the first paragraph of the proof of Theorem~\ref{thm:nonpcGAwAssume}, we see that any $\ell \in \calL$ a run of the EDA (or randomized search heuristic) with parameter $\mu_\ell = 2 U^{\ell-1}$ finds the optimum of $f$ in time $\mu_\ell T$ with probability at least $p / \Pr[\neg A] \ge p$. Consequently, analogous to that proof, the expected runtime conditional on $\neg A$ is at most
\begin{equation}\label{eq:sums}
\sum_{i=1}^{\ell'-1}B_i + p\cdot 2U^{\ell'-1}T + \sum_{i=1}^{{\max \{\calL\}} - \ell'}(1-p)^ip\left(\sum_{j=0}^{i-1} B_{\ell'+j} + 2U^{\ell'+i-1}T\right).
\end{equation}
This expression is identical to the corresponding one in the proof of Theorem~\ref{thm:nonpcGAwAssume} except that the second sum is not taken over the range $i \in \Z_{\ge 1}$, but only the range $i \in [1..{\max \{\calL\}}-\ell']$. Since these sums involve positive terms only, we can bound~\eqref{eq:sums} from above in exactly the same way as in  the proof of Theorem~\ref{thm:nonpcGAwAssume}. This shows our claim.
\end{proof}

\subsection{Specific Runtime Results}\label{ssec:specprobs}

The following examples show how to combine our general runtime analysis with known runtime results to obtain performance guarantees for smart-restart EDAs on several specific problems.
  
\subsubsection{\onemax and \jump}

We recall the following runtime results for the cGA on  \onemax~\citep{SudholtW19} and \jump~\citep{Doerr21cgajump} as well as for the UMDA on \onemax~\citep{Witt19,DangLN19}. As common, by \emph{runtime} we mean the number of fitness evaluations until the optimum is sampled. This is, essentially, two times the number of generations until the optimum is sampled for the cGA, and $\lambda$ times this generation number for the UMDA.

\begin{theorem}[\cite{SudholtW19,Doerr21cgajump,Witt19,DangLN19}]
\label{thm:theoryeda}
Let $K>0$ be a sufficiently large constant and let $C>0$ be any constant. Consider the cGA with $K \sqrt n \ln n \le \mu \le n^C$ and the UMDA with $K \sqrt n \ln n \le \mu \le n^C$ and $\lambda=\Theta(\mu)$~\citep{Witt19} (or the UMDA with $\mu\ge K \sqrt n \ln n$ and $\lambda \ge a \mu$ for sufficiently large constant $a>1$~\citep{DangLN19}).
\begin{itemize}
\item The expected runtimes on the \onemax function are $O(\mu \sqrt n)$ for the cGA~\cite[Theorem~2]{SudholtW19} and $O(\lambda \sqrt n)$ for the UMDA, see~\cite[Theorem~10]{Witt19} and~\cite[Theorem~9]{DangLN19}.
\item With probability $1-o(1)$, the optimum of the \jump function with jump size $k< \tfrac{1}{20} \ln n$ is found by the cGA in time $O(\mu \sqrt n)$ \cite[Theorem~9]{Doerr21cgajump}.
\end{itemize}
With the optimal parameters, that is, with the smallest applicable population sizes, these runtimes are all $O(n \log n)$.
\label{thm:omjump}
\end{theorem}

With Theorem~\ref{thm:nonpcGAwAssume2}, we have the following result.

\begin{theorem}
Consider the smart-restart cGA with update factor $U> 1$ optimizing the \jump function with jump size $k< \tfrac{1}{20} \ln n$ or the \onemax function, or the smart-restart UMDA with update factor $U> 1$ optimizing the \onemax function. Then, apart from a rare event of probability at most $n^{-\omega(1)}$, we have the following estimates for the expected runtime.
\begin{itemize}
\item If the budget factor $b$ is $\Theta(1/\log n)$, then the expected runtime is $O(n \log n)$.
\item If the budget factor $b$ is between $\Omega(1/\log^2 n)$ and $O(1)$, then the expected runtime is $O(n \log^2 n)$.
\end{itemize}
\label{thm:smartomjump}
\end{theorem}

\begin{proof}
For the smart-restart cGA, we note that the \jump result (more {specifically}, the ``with probability $1-o(1)$'' clause) also applies to \onemax simply because the \jump function with jump size $k=1$ has a fitness landscape that can in a monotonic manner be transformed into the one of the \onemax function. Hence for $n$ sufficiently large, we have Assumption~(L') satisfied with $\tilde \mu =  K \sqrt n \ln n$, $\mu^+ = n^C$, $T = O(\sqrt n)$, and $p = 1-o(1)$. Consequently, for any (constant) $U > 1$ we have $p\in (1-\frac{1}{U^2},1)$. Given the above information on $\tilde \mu$ and $T$, we see that any $b \in n^{-o(1)} \cap n^{o(1)}$ gives that $\ell' = \min\{\ell \mid 2U^{\ell-1} \ge \tilde{\mu}, B_{\ell} \ge 2U^{\ell-1} T\} = (1 \pm o(1)) \frac 12 \log_{U}(n)$. Since $\mu^+ = n^C$, we have $|\calL| = (1 \pm o(1)) (C - \frac 12) \log_{U}(n)$. 

For the results where Theorem~\ref{thm:theoryeda} only gives bounds on the expected runtime, we note that such statements can easily be transferred to a statement with a given probability via Markov's bound $\Pr[\xi \ge cE[\xi]] \le 1/c$ for any $c \ge 1$. Consequently, for the smart-restart UMDA with the population size $\lambda$ as parameter of interest, we can set $p\in (1-\frac{1}{U^2},1)$ for any (constant) $U > 1$ to satisfy Assumption~(L') together with $\tilde \mu =  K \sqrt n \ln n$, $\mu^+ = n^C$, and $T = O(\sqrt n)$. Hence, the same arguments as above also show the other claims. 
\end{proof}

Hence, the smart-restart cGA and smart-restart UMDA with $b=\Theta(1/\log n)$ have essentially the same time complexity as the original cGA and UMDA with optimal population size (see Theorem~\ref{thm:omjump}). A constant value for $b$ results in  a slightly inferior runtime of $O(n \log^2 n)$, which is also the runtime guarantee for the parallel-run cGA (Theorem~\ref{thm:parallel}).

\subsubsection{\LO and \DLB}

As discussed in Section~\ref{sssec:basic}, no theoretical runtime guarantees exists for the cGA on the \LO and \DLB functions. For the UMDA, the following results are known.
\begin{theorem}[\cite{DoerrK21ecj,DoerrK21tcs}]
Let $K>0$ and $C > 1$ be sufficiently large constants. Consider the UMDA selecting $\mu \ge K n \ln n$ best from the sampling population with size $\lambda \ge C \mu$.
\begin{itemize}
\item The expected runtime on the \LO function is $O(\lambda n)$~\cite[Theorem~5]{DoerrK21tcs}.
\item With probability $1-o(1)$, the optimum of the \DLB is found in time $O(\lambda n)$ \cite[Theorem~3]{DoerrK21ecj}.
\end{itemize}
With the optimal parameter choice, that is, the smallest admissible population sizes, these runtimes are $O(n^2 \log n)$.
\label{thm:lodlb}
\end{theorem}
For simplicity, we could just require $\lambda=C\mu$ and Theorem~\ref{thm:lodlb} still holds. For $n$ sufficiently large, we have Assumption~(L) with respect to the parameter $\lambda$ satisfied with $\tilde \mu =  CK n \ln n, T=O(n),$ and $p = 1-o(1)$. Consequently, for any (constant) $U > 1$ we have $p\in (1-\frac{1}{U^2},1)$. Given the above information on $\tilde \mu$ and $T$, and with Theorem~\ref{thm:nonpcGAwAssume}, we have the following result.
\begin{theorem}
Consider the smart-restart UMDA with update factor $U> 1$ and constant selection pressure in UMDA optimizing the \LO or \DLB function. Then we have the following estimates for the expected runtime.
\begin{itemize}
\item If the budget factor $b$ is $\Theta(1/\log n)$, then the expected runtime is $O(n^2 \log n)$.
\item If the budget factor $b$ is between $\Omega(1/\log^2 n)$ and $O(1)$, then the expected runtime is $O(n^2 \log^2 n)$.
\end{itemize}
\label{thm:smartlodlb}
\end{theorem}
Hence, our smart-restart UMDA with $b=\Theta(1/\log n)$ has essentially the same time complexity as the original UMDA (Theorem~\ref{thm:lodlb}) with optimal population sizes $\lambda$ and $\mu$. 

\subsubsection{Noisy \onemax}

For another example, we recall the runtime of the cGA (without artificial margins) on the \onemax function with additive centered Gaussian noise from~\citep{FriedrichKKS17}.

\begin{theorem}[{\cite[Theorem~5]{FriedrichKKS17}}]
Let $h : [0,\infty) \to [0,\infty)$ and $h\in\omega(1) \cap n^{o(1)}$. Consider the $n$-dimensional \onemax function with additive centered Gaussian noise with variance $\sigma^2>0$. Then with probability $1-o(1)$, the cGA (without margins) with population size $\mu = h(n) \sigma^2 \sqrt n \log n$ has all frequencies at $1$ in $O(\mu \sigma^2\sqrt n \log (\mu n))=O(h(n) \sigma^4 n \log^2 n)$ iterations.
\label{thm:cgaFKKS}
\end{theorem}

We note that here the cGA is used without restricting the frequencies to the interval $[1/n,1-1/n]$, whereas more commonly (and in the remainder of this paper) the cGA is equipped with the margins $1/n$ and $1-1/n$ to avoid that frequencies reach the absorbing boundaries $0$ or $1$. Since our general runtime results do not rely on such implementation details but merely lift a result for a particular cGA to its smart-restart version, this poses no problems for us. As a side remark, though, we note that we are very optimistic that the above result from~\citep{FriedrichKKS17} holds equally for the setting with frequency margins. 

More interestingly, the runtime result above is not of the type that for $\mu$ sufficiently large, the expected runtime is $O(\mu T)$ for some $T$ (since $\mu$ appears also in the $\log(\mu n)$ term). Fortunately, with Theorem~\ref{thm:nonpcGAwAssume2} at hand, we have an easy solution. By only regarding values of $\mu$ that are at most $n^C$ for some constant $C$ (which we may choose), the $\log(\mu n)$ term can by bounded by $O(\log n)$. Since the minimal applicable $\mu$ (the $\tilde \mu$ in the notation of Theorem~\ref{thm:nonpcGAwAssume2}) depends on $\sigma^2$, this also implies that we can only regard polynomially bounded variances, but it is clear that any larger variances can be only of a purely academic interest. We thus formulate and prove the following result. 
We note that with more work, we could also have extended Theorem~\ref{thm:nonpcGAwAssume} to directly deal with the runtime behavior described in Theorem~\ref{thm:cgaFKKS}. For example, we could exploit that the geometric series showing up in the analysis do not change significantly when an extra logarithmic term is present. However, this appears to be a lot of work for a logarithmic term for which it is not even clear if it is necessary in the original result. Hence, we will not discuss them and focus on the following result.

\begin{theorem}
Let $C\ge 1$ and $U > 1$. Let $h : [0,\infty) \to [0,\infty)$ and $h\in \omega(1) \cap n^{o(1)}$. Consider the smart-restart cGA with the update factor $U$ and budget factor $b$ optimizing the $n$-dimensional \onemax function with additive centered Gaussian noise with variance $\sigma^2 \le n^C$. Then outside a rare event holding with probability $n^{-\omega(1)}$, the following runtime estimates are true.
\begin{itemize}
\item If $b=\Theta(1/h(n))$, then the expected runtime is $O(h(n) \sigma^4 n \log^2 n)$.
\item If $b = O(1) \cap \Omega(1/\log^2 n)$, then the expected runtime is $O(h(n) \sigma^4 n \log^3 n)$.
\end{itemize}
\label{thm:noisyom}
\end{theorem}

\begin{proof}
  By Theorem~\ref{thm:cgaFKKS}, we have Assumption (L') satisfied with $\tilde \mu = h(\mu) \sigma^2 \sqrt n \ln(n)$, $\mu^+ = n^{2C}$, $T = O(\sigma^2 \sqrt n \log n)$, and $p = 1-o(1)$. Consequently, any $b \in n^{-o(1)} \cap n^{o(1)}$ gives that $\ell' = \min\{\ell \mid 2U^{\ell-1} \ge \tilde{\mu}, B_{\ell} \ge 2U^{\ell-1} T\} = (1 \pm o(1)) (\frac 12 \log_{U}(n) + \log_{U}(\sigma^2)) \le (1+o(1)) C \log_{U}(n)$. Since $\mu^+ = n^{2C}$, we have $|\calL| \ge (1 \pm o(1)) C \log_{U}(n)$. With Theorem~\ref{thm:nonpcGAwAssume2}, we have proven our claim.
\end{proof}

We remark that the parallel-run cGA has an expected runtime of $O(h(n) \sigma^4 n \log^3 n)$ outside a rare event of probability $n^{-\omega(n)}$.

\section{Experimental Results}
\label{sec:exper}

In this section, we experimentally analyze the smart-restart mechanism proposed in this work. We concentrate on the smart-restart cGA, since the cGA is one of the best-unterstood EDAs and since it has only a single parameter which governs exactly how strong the update of the probabilistic model is in each iteration. We regard PBIL as a more complex EDA in Section~\ref{sec:pbil}.

Since, apart from the analysis of the cGA on \onemax by Lengler et al.~\cite{LenglerSW21}, such data is not yet available, we start with an investigation of how the runtime of the original cGA depends on the (hypothetical) population size $\mu$. This will in particular support the assumption, underlying our smart-restart strategy and the parallel-run strategy from~\citep{Doerr21cgajump}, that the runtime can be excessively large when $\mu$ is below some threshold, and moderate and linearly increasing with $\mu$ when $\mu$ is larger than this threshold.

We then analyze the performance of the two existing approaches to automatically find good values for~$\mu$. Our focus is on understanding how one can relieve the user of an EDA from the difficult task of setting this parameter, not on finding the most efficient algorithm for the benchmark problems we regard. For this reason, we do not include other algorithms in this investigation. We note, though, that EDAs have shown a superior performance on the \jump and \DLB benchmarks~\citep{HasenohrlS18,Doerr21cgajump,WangZD21}, so clearly these are interesting algorithms for these two problems.

\subsection{Experimental Settings}

We ran the original cGA (with varying population sizes), the parallel-run cGA, and our smart-restart cGA (with two budget factors) on four benchmark problems, both without noise and in the presence of Gaussian posterior noise of four different strengths. For each experiment for the parallel-run cGA and our smart-restart cGA, we conducted 20 independent trials. Due to the often extremely large runtimes in the regime with genetic drift, only 10 independent trials were conducted for the original cGA. The detailed settings for our experiments were as follows.
\begin{itemize}
\item Benchmark functions and problem sizes: \onemax (problem size $n=100$), \LO ($n=50$), \jump ($n=50$ and jump size $k=10$), and \DLB ($n=30$). The population size for \onemax was chosen identical to the one used in~\citep{FriedrichKKS17}, namely $n = 100$. For the other three problems, taking into account the longer runtimes, we chose relatively small problem sizes. Since for these our experimental results fit the known theoretical results (see Sections~\ref{ssec:benchmarks} and~\ref{ssec:specprobs}), we are confident that they are still representative.
\item Noise model: additive centered Gaussian posterior noise with variances ${\sigma^2=\{0,n/2,}$ ${n,2n,4n\}}$ as described in Section~\ref{ssec:intronoise}.
\item Termination criterion: Since the original cGA with unsuitable population sizes often did not find the optimum in a reasonable time, we imposed the following maximum numbers of generations and aborted the run after this number of generations: $\lceil n^4\ln n \rceil$ for \onemax, $n^5$ for \LO, $n^{k/2}$ for \jump, and $10n^5$ for \DLB. We did not define such a termination criterion for the parameter-less versions of the cGA since they always found the optimum in an affordable time.
\item Population size of the original cGA: $\mu=2^{[1..10]}$ for \onemax and \LO, $\mu=2^{[9..18]}$ for \jump, and $\mu=2^{[1..14]}$ for \DLB. 
The reason for omitting the range $\mu=2^{[1..8]}$ for \jump is the large runtime observed on this benchmark for small population sizes.
\item Budget factor $b$ for the smart-restart cGA: $16$ and $1/\ln n$. As explained in the introduction, the budget factors $b=16$ and $\Theta(1 / \ln n)$ are two proper choices. We chose the constant $1$ based on the experimental results on \jump and \DLB without noise (noise variance $\sigma^2=0$) in Figures~\ref{fig:jump} and~\ref{fig:dlb}. 
\item Update factor $U$ for the smart-restart cGA: $2$. Doubling the parameter value after each unsuccessful run ($U=2$) is a natural choice. We note that in~\citep{DoerrZ20gecco}, we also did some experiments with $U = \sqrt 2$, but these mostly gave inferior results. 

\end{itemize}

\subsection{Experimental Results and Analysis I: The cGA with Different Population Sizes}
 
The curves in Figures~\ref{fig:om}--\ref{fig:dlb} (excluding the three right-most points on the $x$-axis) show the runtime (number of fitness evaluations) of the original cGA with different population sizes when optimizing our four benchmarks under Gaussian noise with different variances (including the noise-free setting $\sigma^2=0$). Given are the median runtime together with the first and third quartiles. When a run was stopped because the maximum number of function evaluations was reached, we simply and bluntly counted the runtime up to this point as runtime. Clearly, there are better ways to handle such incomplete runs, but since a fair computation for these inefficient parameter ranges is not too important, we did not start a more elaborate evaluation. To ease the comparison, we also plotted the run budgets $B = b \mu^2$ which the smart-restart algorithm with budget factor $b$ would have with population size~$\mu$.

\begin{figure}[!ht]
\centering
\includegraphics[width=5.0in]{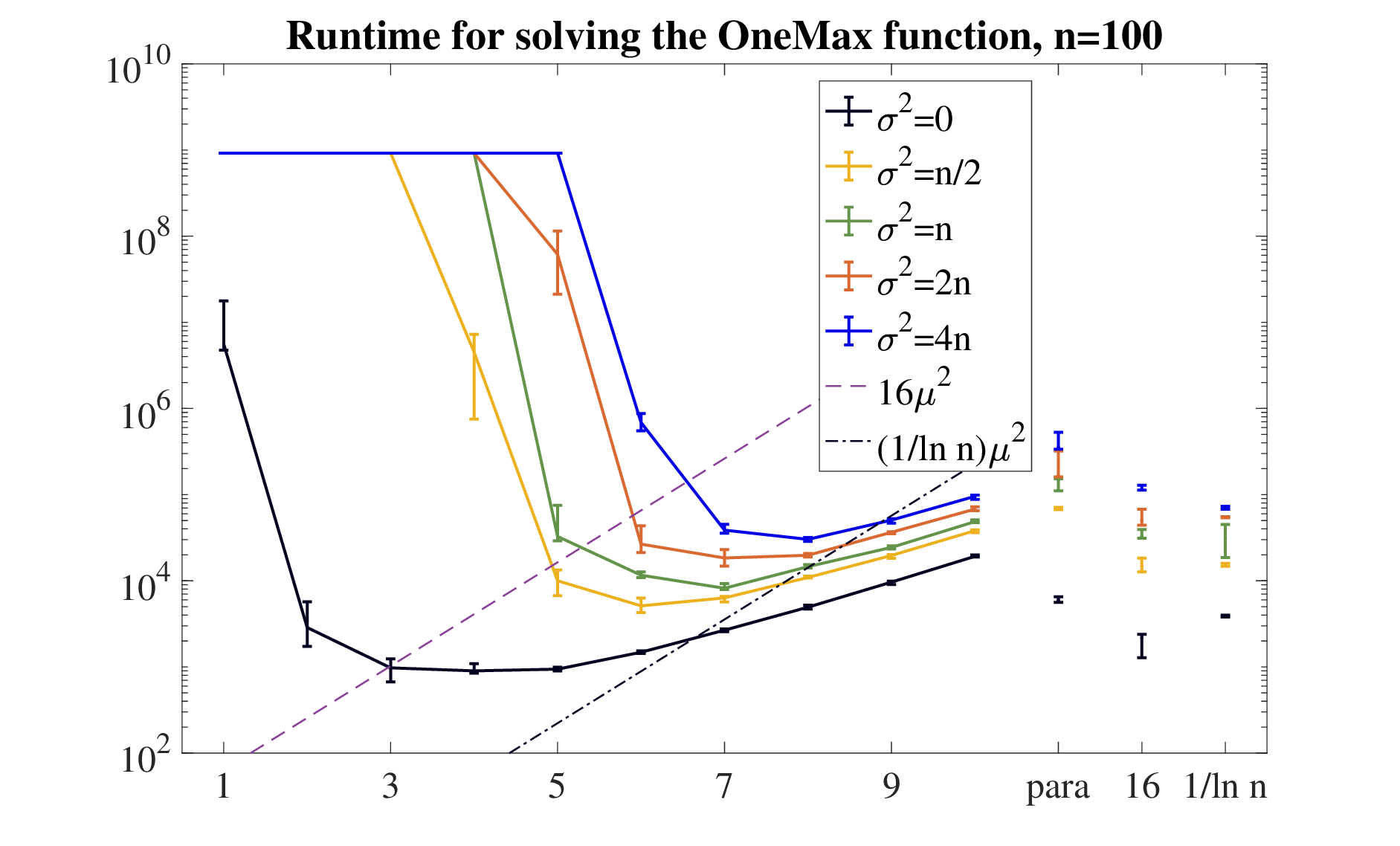}
\caption{The median number of fitness evaluations (with the first and third quartiles) of the original cGA with different $\mu$ ($\log_2 \mu \in \{1,2,\dots,10\}$), the parallel-run cGA (``para''), and the smart-restart cGA with two budget factors ($b=16$ and $b=1/\ln n$) on the \onemax function ($n=100$) under Gaussian noise with variances $\sigma^2=0, n/2, n, 2n, 4n$ in 20 independent runs (10 runs for the original cGA). }
\label{fig:om}
\end{figure} 

\begin{figure}[!ht]
\centering
\includegraphics[width=5.0in]{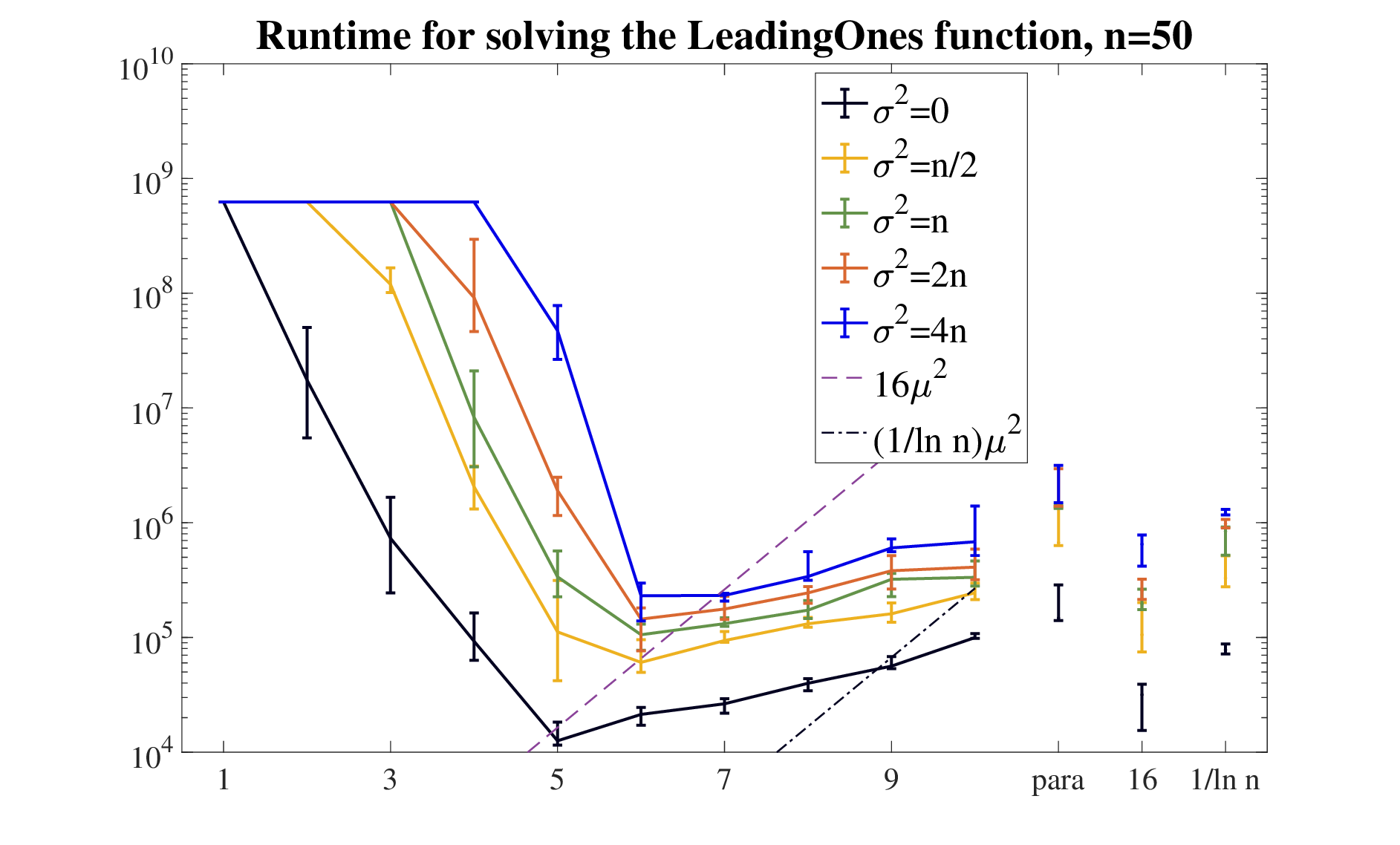}
\caption{The median number of fitness evaluations (with the first and third quartiles) of the original cGA with different $\mu$ ($\log_2 \mu \in \{1,2,\dots,10\}$), the parallel-run cGA (``para''), and the smart-restart cGA with two budget factors ($b=16$ and $b=1/\ln n$) on the \LO function ($n=50$) under Gaussian noise with variances $\sigma^2=0, n/2, n, 2n, 4n$ in 20 independent runs (10 runs for the original cGA). }
\label{fig:lo}
\end{figure} 

\begin{figure}[!ht]
\centering
\includegraphics[width=5.0in]{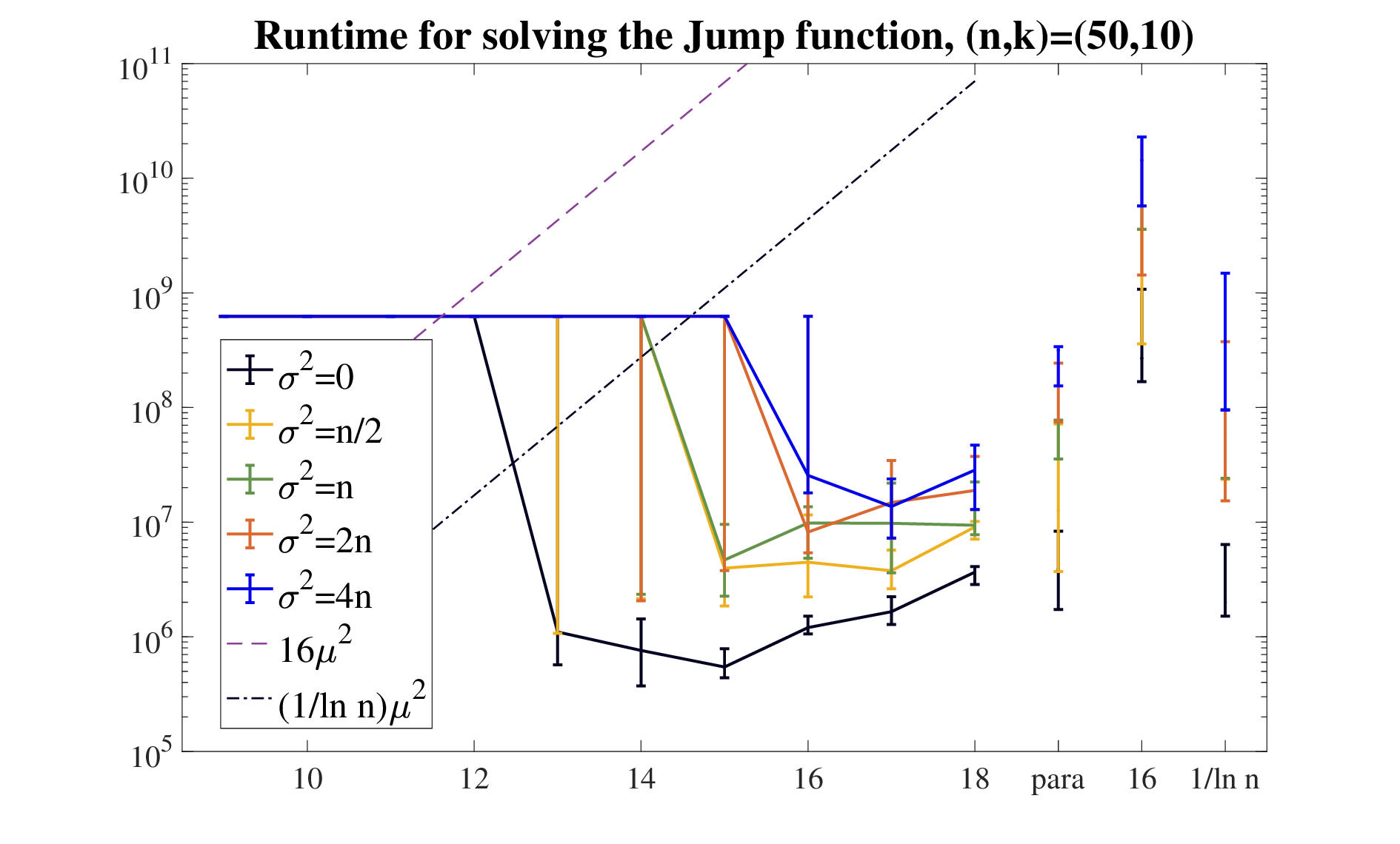}
\caption{The median number of fitness evaluations (with the first and third quartiles) of the original cGA with different $\mu$ ($\log_2 \mu \in \{9,10,\dots,18\}$), the parallel-run cGA (``para''), and the smart-restart cGA with two budget factors ($b=16$ and $b=1/\ln n$) on the \jump function with $(n,k)=(50,10)$ under Gaussian noise with variances $\sigma^2=0, n/2, n, 2n, 4n$ in 20 independent runs (10 runs for the original cGA). }
\label{fig:jump}
\end{figure} 

\begin{figure}[!ht]
\centering
\includegraphics[width=5.0in]{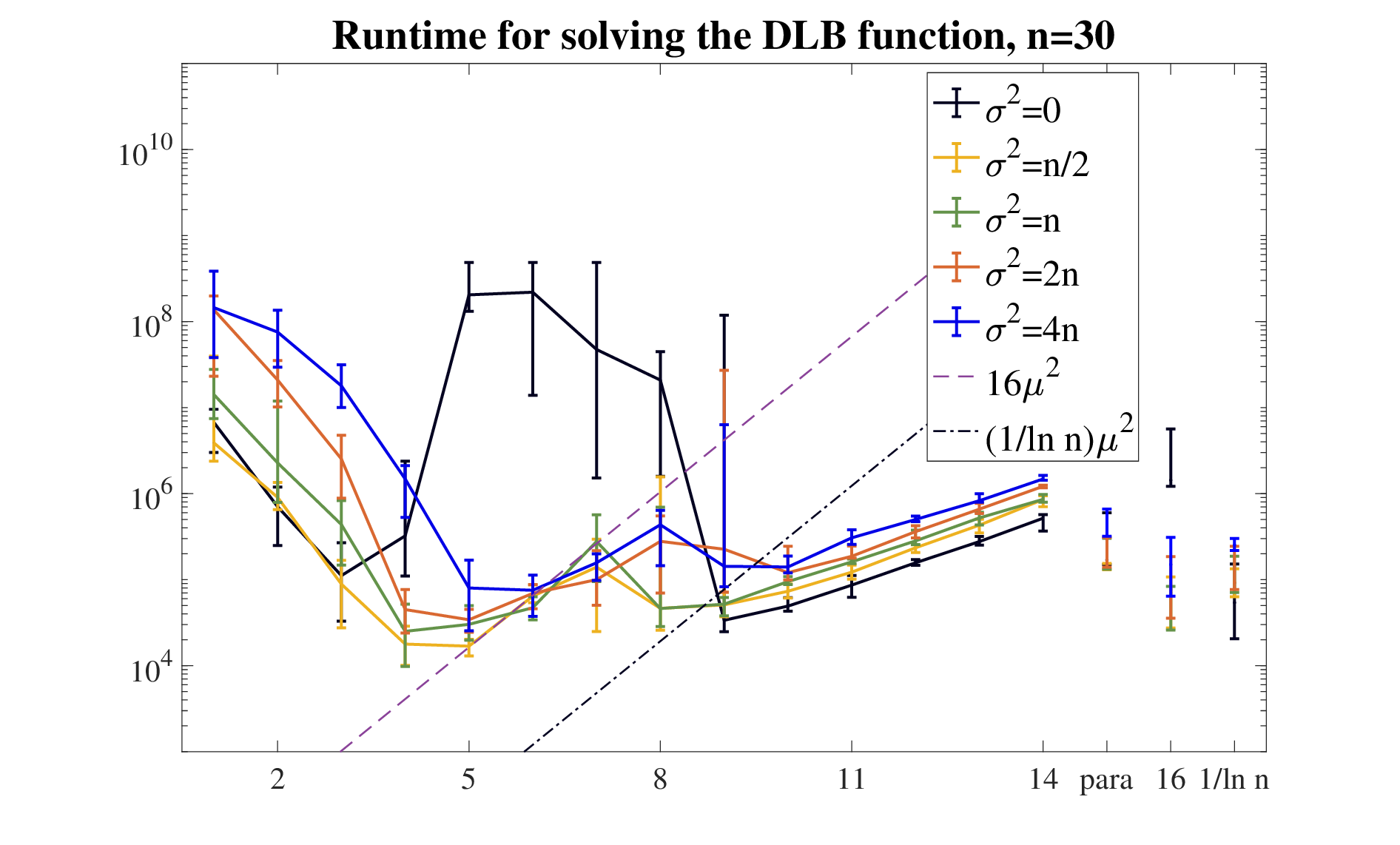}
\caption{The median number of fitness evaluations (with the first and third quartiles) of the original cGA with different $\mu$ ($\log_2 \mu \in \{1,2,\dots,14\}$), the parallel-run cGA (``para''), and the smart-restart cGA with two budget factors ($b=8$ and $b=0.5/\ln n$) on the \DLB function ($n=30$) under Gaussian noise with variances $\sigma^2=0, n/2, n, 2n, 4n$ in 20 independent runs (10 runs for the original cGA). }
\label{fig:dlb}
\end{figure}


The results displayed in Figures~\ref{fig:om}--\ref{fig:dlb} typically show that the runtime of the cGA is roughly unimodal in the population size~$\mu$. For values of $\mu$ smaller than the optimal value, the runtime steeply increases and is accompanied by larger variances. For larger values of $\mu$, we typically observe a moderate, roughly linear increase of the runtime. The variances are relatively small here. 

Let us regard these results in some more detail. For \onemax, we see a perfect unimodal runtime behavior. The minima of the runtime curves are not strongly pronounced, but for population sizes smaller than the optimum one by a factor of four or more, drastic performance losses are observed. For larger population sizes, a roughly linear increase of the runtime is well visible. We also note that the optimal population size increases with the noise level, which fits the intuition that larger noise levels lead to larger runtimes, which need larger population sizes to prevent genetic drift throughout the runtime.  

Our experiments do not show the bimodal runtime behavior observed in~\citep{LenglerSW21}. This is not suprising given that the bimodal pattern is very weak. In Figure~2 in~\citep{LenglerSW21}, the runtime pattern for relatively large problem size $n=1000$ shows a global minimum at $\mu \approx 130$ (where the genetic drift is low as also shown in that figure). There is a second (local) minimum at around $\mu \approx 12$, but its runtime is only around 4\% smaller than the runtime at the local maximum between the two minima. Given these small differences, seen for problem size $n=1000$ in $3000$ independent runs, it is not surprising that our experiments, conducted for smaller problem size and larger ranges of $\mu$ (resulting in much larger ranges of the runtimes), cannot detect this bimodal runtime behavior.

For \leadingones, we observe a more pronounced optimal value for $\mu$. Reducing $\mu$ below this level leads to a clear increase of the runtime, typically at least by a factor of $10$ for each halving of $\mu$ (which is still less drastic than for \onemax or \jump functions). Interestingly, the optimal population size is relatively independent from the noise level (say compared to the \onemax results). We have no explanation for this. 

For \jump functions, the optimal $\mu$-value is again less pronounced, however reducing the population size $\mu$ below the efficient values immediately gives a catastrophic increase of the runtime, almost always leading to all runs being stopped because the maximum number of generations is reached. In the very narrow transition regime between efficient and catastrophic optimization, we observe large variances of the runtime. This could indicate that there is not a continuous increase of the typical runtime, but rather an increase of the probability that a run enters an unfavorable situation, e.g., caused by genetic drift. 

We recall that we ran the parameter-less versions of the algorithm without a termination criterion (since they always found the optimum), so it is for this reason that some of these runtimes appear larger than those for small static values of $\mu$ (which were just stopped after $n^{k/2} = 50^5 = 312{,}500{,}000$ iterations, i.e., $6.25 \cdot 10^8$ fitness evaluations, when not optimum was found before that).

The runtime behavior on \DLB is harder to understand. There is a clear ``linear regime'' from $\mu=2^9$ or $\mu=2^{10}$ on, again with very small variances. There is also a steep increase of the runtimes roughly starting at $\mu2^4$. In between these two regimes, the runtime behavior is hard to understand. The noisy runs show a small increase of the runtime in this middle regime together with slightly increased variances. The noise-free runs, however, are massively slower than the noisy ones, with large variances and a decent number of unsuccessful runs. We have no explanation for this. 

Apart from the runtimes on \DLB (though to some extent also here, namely in the noisy runs), our results indicate a runtime behavior as described in Assumption~(L). We have no proof for the fact that this behavior is caused by the effect of genetic drift and such a proof is most likely not easy to give. For this work, however, such a proof is not indispensable -- what counts is that our understanding of genetic drift led to the development of the smart-restart scheme, which both in mathematical runtime analyses and in experiments showed a good performance and this is without the need to tune the hypothetical population size of the cGA (which is, in turn, indispensable when using static parameters as shown by our experiments).
%

As a side result, this data confirms that the cGA has a good performance on noise-free \jump functions, not only in asymptotic terms as proven in~\citep{HasenohrlS18,Doerr21cgajump}, but also in terms of actual runtimes for concrete problem sizes. On a \jump function with parameters $n=50$ and $k=10$, a classic mutation-based algorithm would run into the local optimum and from there would need to generate the global optimum via one mutation. For standard bit mutation with mutation rate $\frac 1n$, this last step would take an expected time of $n^{k} (\frac {n}{n-1})^{n-k}$, which for our values of $n$ and $k$ is approximately $2.2 \cdot 10^{17}$. With the asymptotically optimal mutation rate of $\frac kn$ determined in~\citep{DoerrLMN17}, this time would still be approximately $7.3 \cdot 10^{10}$. In contrast, the median optimization time of the cGA with $\mu \in 2^{[15..18]}$ is always below $4 \cdot 10^{6}$.

Our data also indicates that a good performance of the cGA can often be obtained with much smaller population sizes (and thus more efficiently) than what previous theoretical works suggest. For example, in \citep{FriedrichKKS17} a population size of $\omega(\sigma^2 \sqrt n \log n)$ was required for the optimization of a noisy \onemax function via the cGA. In their experiments on a noisy \onemax function with $n=100$ and ${\sigma^2=n}$, a population size (called $K$ in~\citep{FriedrichKKS17} to be consistent with some previous works) of $\mu = 7\sigma^2 \sqrt n (\ln n)^2\approx 148{,}000$ was used, which led to a runtime of approximately $200{,}000$ (data point for $\sigma^2 = 100$ interpolated from the two existing data points for $\sigma^2 = 64$ and $\sigma^2 = 128$ in the left chart of Figure~1 in~\citep{FriedrichKKS17}).
In contrast, our experiments displayed in Figure~\ref{fig:om} suggest that population sizes between 64 and 256 are already well sufficient and give runtimes clearly below $20{,}000$.

We have to admit that we do not fully understand this number $200{,}000$ from~\citep{FriedrichKKS17} and expect that it should be much larger. Our skepticism is based both on theoretical and experimental considerations. On the theoretical side, we note that even in the absence of noise and with the frequency vector having the (for this purpose) ideal value $\tau = (\frac 12, \dots, \frac 12)$, the sum $\|\tau\|_1$ of the frequency values increases by an expected value of $O(\frac 1\mu \sqrt n)$ only (with small leading constant; an absolute upper bound of $\frac 1{2\mu} \sqrt n$ follows, e.g., easily from~\citep{BerendK13}). Hence after only $200{,}000$ iterations, the frequency sum $\|\tau\|_1$ should still be relatively close to $n/2$. Since the probability to sample the optimum is $\prod_{i=1}^n (1-\tau_i) \le \exp(-\|\tau\|_1)$, it appears unlikely that the optimum is sampled within that short time. Our experimental data displayed in Figure~\ref{fig:om} suggests an affine-linear dependence of the runtime on $\mu$ when $\mu$ is at least $2^8$. From the median runtimes for $\mu=2^9$ and $\mu = 2^{10}$, which are $T_9 = 24{,}384$ and $T_{10} = 48{,}562$, we would thus estimate a runtime of $T(\mu) = T_9 + (T_{10}-T_9) (\mu - 2^9) 2^{-9}$ for $\mu \ge 2^{10}$, in particular, $T(7\sigma^2 \sqrt n (\ln n)^2) = 7{,}010{,}551$ for the data point $\sigma^2 = 100$ and $n=100$. 
To resolve this discrepancy, we conducted $20$ runs of the cGA with $\mu = \lfloor 7\sigma^2 \sqrt n (\ln n)^2 + \frac 12 \rfloor$, $\sigma^2 = 100$, $n = 100$ and observed a median runtime of 5,728,969 (and a low variance, in fact, all 20 runtimes were in the interval $[5{,}042{,}714; 6{,}131{,}522]$). So most likely, the number of $200{,}000$ given in~\citep{FriedrichKKS17} is not correct, and the price for the large value of $\mu$ is significantly larger than what the $200{,}000$ suggests.

\subsection{Experimental Results and Analysis II: Runtimes of the Parallel-Run cGA and the Smart-Restart cGA}

The three right-most items on the $x$-axis in Figures~\ref{fig:om}--\ref{fig:dlb} show the runtimes of the parallel-run cGA and the smart-restart cGA (with two budget factors~$b$). Figures~\ref{fig:om}--\ref{fig:dlb} also plot the two budgets (number of fitness evaluations) $16\mu^2$ and $(1/\ln n)\mu^2$ corresponding to $b=16$ and $1/\ln n$ respectively. 

The intersection point of the runtime curve of the cGA and the budget curve is a good indication for the $\mu$-value with which the smart-restart cGA finds the optimum, and thus it is also a good indication for the runtime of the smart restart cGA. For example, the smallest $\mu$ such that the $16\mu^2$ curve is above the noiseless \leadingones curve is $\mu=2^5$. Consequently, we expect the smart-restart cGA with budget factor $16$ to not find the optimum of the noiseless \leadingones functions in the runs with $\mu = 2, 4, 8, 16$, but only in the run with $\mu=32$. For this reason, the runtime of this smart-restart cGA should be equal to the runtime of the original cGA with $\mu=32$ plus $16\cdot 2^2 + 16\cdot 4^2 + 16 \cdot 8^2 + 16 \cdot 16^2 = 5440$, which fits roughly to our experimental data.

We note that for most runtime curves of the classic cGA, the intersection points with the budget curves are in the linear regime, which means that the corresponding smart-restart cGA avoids spending much time in the inefficient genetic drift regime. Some intersection points, e.g., those for the $(1/\ln(n))\mu^2$ budget in the \leadingones figure, are far in the linear regime. This indicates that the corresponding smart restart cGA misses the better (smaller) $\mu$-values that are already outside the genetic drift regime. As the curves for the original cGA and the performances of the smart-restart cGA show, the performance loss of this miss is not too large. It is clearly much less than the catastrophic performance loss from running the cGA in the regime with strong genetic drift.

In more detail, we see that for the easy functions \onemax and \LO under all noise assumptions, the smart-restart cGA with both values of $b$ has a smaller runtime than the parallel-run cGA. This can be explained from the runtime data of the original cGA in the corresponding figures: Since the runtimes are similar for several population sizes, the parallel-run cGA with its strategy to assign a similar budget to different population sizes wastes computational power, which the smart-restart cGA saves by aborting some processes early and not starting others. For both functions, the larger budget factor typically is superior. This fits again to the data on the original cGA, where we see the smaller budget factor curve intersecting the runtime curve clearly in the linear regime.

More interesting are the results for \jump and \DLB. We recall that here a wrong choice of the population size can be catastrophic, so these are the two functions where not having to choose the population size is a big advantage for the user. What is clearly visible from the data is that here the smaller budget factor is preferable for the smart-restart cGA. This fits our previously gained intuition that for these two functions, genetic drift is detrimental. Hence there is no gain from continuing a run that is suffering from genetic drift (we note that there is no way to detect genetic drift on the fly -- a frequency can be at a (wrong) boundary value due to genetic drift or at a (correct) boundary value because of a sufficiently strong fitness signal). 

What is clear as a general rule is that both algorithms, the parallel-run cGA and the smart-restart cGA with the small fitness evaluation budget factor, clearly do a good job in successfully running the cGA with a reasonable population size -- recall that for both of the difficult functions, a wrong choice of the population size can easily imply that the cGA does not find the optimum in $10^8$ iterations.

\section{Smart-Restart PBIL (Cross-Entropy Algorithm)}\label{sec:pbil}

To see how smart-restart EDAs perform on combinatorial optimization problems and also to discuss a third EDA in this work, we now conduct an experimental analysis of \emph{population-based incremental learning (PBIL)} (cross-entropy algorithm) on two optimization problems it was applied to in the literature. For comparison, two other restart strategies originally designed for evolutionary algorithms are also adapted to PBIL and implemented.

\subsection{Smart-Restart Population-Based Incremental Learning (Smart-Restart Cross-Entropy)}
\label{ssec:srPBIL}

Besides the cGA and UMDA, in~\citep{DoerrZ20tec} also a theoretical analysis of the boundary hitting time caused by genetic drift in the algorithm PBIL~\citep{Baluja94,BalujaC95} was conducted. This algorithm is identical to the basic version of the cross-entropy~(CE) algorithm for discrete optimization~\citep{CostaJK07,BoerKMR05}.\footnote{We point out a possible tiny difference between PBIL and CE. According to the algorithm description of the CE algorithm in~\citep{CostaJK07} and in the textbook~\cite[Algorithm~2.4.1]{RubinsteinK04}, the CE algorithm selects all individuals with fitness at least the fitness of the $\mu$-th best solution for the model update, whereas PBIL selects exactly $\mu$ best solutions breaking possible ties at random. However, the code provided in~\cite[Page 275]{RubinsteinK04} and also the recent pseudocode in~\cite[Algorithm~1]{WuKM17} both have a fixed cutoff (and they do not discuss the problem of tie-breaking; more precisely, the tie-breaking is determined by how the sorting routine breaks the ties; since we are talking about random samples, it is clear anyway that the tie-breaking is not important). Given this state of the art, we prefer to think of the CE algorithm as also working with a fixed cutoff, and thus say that PBIL and CE are identical, as also said in~\citep{KrejcaW20}.} It further includes the UMDA~\citep{MuhlenbeinP96} and the $\lambda$-max-min ant system ($\lambda$-MMAS)~\citep{StutzleH00}, a classic ant colony optimization algorithm, as special cases. 

The general procedure of PBIL is to sample $\lambda$ individuals and select $\mu$ best individuals to learn the current probabilistic model with learning rate $\rho$. An alternative parameterization, which we shall also prefer to ease the comparison with previous works, is to have as parameters (besides the learning rate $\rho$) the sample size $\lambda$ and the selection pressure $\eta$, which define $\mu$ via $\mu = \lceil \eta \lambda \rceil$. Similar to the other EDAs regarded in this work, we also use the artificial margins $\{1/n,1-1/n\}$ to prevent a premature convergence. The pseudocode of this algorithm is given in Algorithm~\ref{alg:PBIL}. 

\begin{algorithm}[!ht]
\caption{The algorithm PBIL (with learning rate $\rho$, sample size $\lambda$, and selection pressure $\eta$) to maximize a function $f: \{0,1\}^n \rightarrow \R$}
{\small
 \begin{algorithmic}[1]
 \STATE{$p^0=(\tfrac{1}{2}, \tfrac{1}{2},\dots,\tfrac{1}{2})\in [0,1]^n$}
 \FOR {$g=1,2,\dots$}
 \STATEx {$\quad\%\%$\textsl{Sample $\lambda$ individuals $X_1^g,\dots,X_{\lambda}^g$}}
 \FOR {$i=1,2,\dots,\lambda$}
 \FOR {$j=1,2,\dots,n$}
 \STATE $X_{i,j}^g \leftarrow 1$ with probability $p_{j}^{g-1}$ and $X_{i,j}^g \leftarrow 0$ with probability $1-p_{j}^{g-1}$
 \ENDFOR
 \ENDFOR
 \STATEx {$\quad\%\%$\textsl{Update of the frequency vector}}
 \STATE Let $\tilde{X}_1^g,\dots,\tilde{X}_{\mu}^g$ be the best $\mu = \lceil \eta \lambda \rceil$ individuals (ties broken randomly)
 \STATE {$p'=\frac{\rho}{\mu} \sum_{i=1}^{\mu}\tilde{X}_i^g + (1-\rho)p^{g-1}$}
 \STATE {$p^g=\min \{\max\{\tfrac{1}{n},p'\},1-\tfrac{1}{n}\}$}
 \ENDFOR
 \end{algorithmic}
 \label{alg:PBIL}
}
\end{algorithm}

Recall that~\cite[Theorem~3]{DoerrZ20tec} proved that for PBIL, the frequency of a neutral bit moves out of the interval $(c\frac{\rho}{\mu},1-c\frac{\rho}{\mu})$ for a constant $c\in(\frac12,\frac{1}{\sqrt2})$ in at most $\frac{16}{2-1/c}\frac{\mu}{\rho^2}$ generations in expectation, that is, at most $\frac{16}{2-1/c}\frac{\mu\lambda}{\rho^2}$ fitness evaluations in expectation. With the notation of the selection pressure $\eta=\mu/\lambda$ and by Markov's inequality, the probability that a boundary is reached in $b\lambda^2, b>\frac{16\eta}{(2-1/c)\rho^2}$, fitness evaluations, is at least $1-\frac{16\eta}{(2-1/c)\rho^2b}$. Hence, it fits into the framework of the smart-restart mechanism as discussed in 
Section~\ref{sec:parameterlesscGA}. We thus obtain a smart-restart PBIL by letting the smart-restart mechanism control the parameter $\lambda$ and keeping the parameters $\rho$ and $\eta$ fixed. 
%

\subsection{Other Restart Strategies\label{ssec:other}}
In Section~\ref{sec:intro}, we mentioned several other generic strategies to remove the population size as a parameter of an algorithm. The two most interesting restart strategies among these shall be included in our experimental investigation, namely the classical strategy from~\citep{HarikL99} (we only consider the first strategy of that work as the other strategy is a parallel-run strategy) and the well-cited strategy of Auger and Hansen~\cite{AugerH05}. Both restart strategies are not originally designed for an EDA, hence we need to adapt them to the PBIL. 

The first strategy of Harik and Lobo~\cite{HarikL99} restarts a crossover-based genetic algorithm when all individuals have become identical. As a reasonable analog for PBIL, we take the criterion that all entries of the probabilistic model $p^g$ are at the boundaries or close to them. More precisely, when using the frequency margins $\{ 1/n, 1- 1/n\}$, the restart criterion is that for all $i=1,\dots,n$, we have $p^g_i\in \{1/n,1-1/n\}$. For PBIL without margins, we conduct a restart when $p^g_i\in (0,1/n^2) \cup (1-1/n^2,1)$ for all $i = 1, \dots, n$. Note that we cannot set all $p^g_i\in\{0,1\}$ as the criterion for PBIL without margins, since the frequencies in a run of PBIL never reach (but approach) $0$ or $1$ for sufficiently large $n$. 


The strategy of Auger and Hansen~\cite{AugerH05} contains five criteria to trigger a restart with a larger population size for the CMA-ES. Four of the five relate to the covariance matrix or the evolution path and thus are specific to CMA-ES. For these, we did not find a natural analog for EDAs. Hence in our analysis for the PBIL, we only discuss their criterion related to the fitness. Here a restart is triggered if the range of the best fitnesses among the last $L=10+\lceil 30n/\lambda \rceil$ iterations is zero or the range of these best fitnesses and the best fitnesses in the current iteration is below a predefined threshold $\epsilon$. 
Algorithm~\ref{alg:other} gives the pseudocode for these two restart strategies (adapted to PBIL). We use HL and AH as shorthands for the two strategies.

\begin{algorithm}[!ht]
\caption{Two adapted restart mechanisms from~\citep{HarikL99,AugerH05} with update factor $U$ applied to an algorithm $\calA$ with parameter $\lambda$ for the maximization of a function $f: \{0,1\}^n \rightarrow \R$. }
{\small
 \begin{algorithmic}[1]
  \FOR { $\ell=1,2,\dots$}
 \STATE Run $\calA$ with parameter value $\lambda_{\ell}=2U^{\ell-1}$ for arbitrarily long time on the maximization problem~$f$ until one of the following situations happens. 
 \begin{itemize}
 \item For HL, at a certain iteration $g$, for all $i=1,\dots,n$, $p^g_i\in \{1/n,1-1/n\}$ for $\calA$ with margins or $p^g_i\in (0,1/n^2) \cup (1-1/n^2,1)$ for $\calA$ without margins
  \item For AH, at a certain iteration $g$, $\max V- \min V=0$ or $\max V'-\min V' < \eps$ where $V=\{v^{g-L},\dots,v^{g-1}\}, V'=V\cup V^g$, $L$ is the predefined memory size, $v^{t}$ is the best fitness value at generation $t$, and $V^g$ is the samples at the generation $g$
 \end{itemize}
  \ENDFOR
 \end{algorithmic}
 \label{alg:other}
}
\end{algorithm}

We note a central difference between our restart strategy (Algorithm~\ref{alg:nonpcGA}) and the two from Algorithm~\ref{alg:other}. Our smart-restart strategy stops each parameter trial when a prespecified computational budget (based on a general understanding of genetic drift in EDAs) is reached. In contrast, HL and AH try to detect on the fly when a situation is reached from which further progress appears difficult. 

\subsection{Experiments on Two Combinatorial Problems}

We empirically test the smart-restart PBIL on two combinatorial problems, namely the max-cut problem and the bipartition problem in the settings used in the textbook~\citep{RubinsteinK04}. 
We recall that our focus is to understand how effective the smart-restart scheme is in finding an efficient value for the sample size~$\lambda$. For that reason, we conduct experiments with the PBIL/CE algorithm as proposed in~\citep{RubinsteinK04} and with our smart-restart version of it, but we do not regard other EDAs. We also regard the two other restart mechanisms discussed above.

\subsubsection{Optimization Problems}

Now we brief{}ly introduce the two problems regarded in this section. Both were used in the textbook~\citep{RubinsteinK04} to demonstrate the power of the CE algorithm. 

In the \emph{max-cut problem}, the input consists of an undirected graph $G = (V,E)$ together with edge weights $w : E \to \R$. The target is to find a partition $(V_1,V_2)$ of the node set $V$ such that the sum of the weights of the edges from $V_1$ to $V_2$ is maximized. The max-cut problem is NP-complete and APX-hard. The best known approximation algorithm is a $0.878$-approximation algorithm based on semi-definite programming by Goemans and Williamson~\cite{GoemansW95}.

In the \emph{bipartition problem}, we are given the same input data, but now in addition the sizes of $V_1$ and $V_2$ are prescribed. Again, the target is to maximize the sum of the weights of the edges from $V_1$ to $V_2$. This problem is again NP-complete and APX-hard. An approximation algorithm with minimally weaker approximation ratio than the Goemans-Williamson algorithm for the max-cut problem was given by \cite{AustrinBG16}.

Both problems can easily be modeled as pseudo-Boolean optimization problems. For both, Rubinstein and Kroese~\cite{RubinsteinK04} propose synthetic problem instances with clear structures so that the unique optimal solution is known in advance. The precise details are not important to understand the remainder, so we omit the details and refer the interested reader to~\cite[Pages 46-49]{RubinsteinK04} for the max-cut problem and~\cite[Pages 145-147]{RubinsteinK04} for the bipartition problem. 

\subsubsection{Experimental Settings}

In all our experiments, we use the problems as proposed and modeled in~\cite{RubinsteinK04}. In our implementation of the core PBIL, we use the Matlab code provided in~\cite[Pages~274-276]{RubinsteinK04}. This code is formulated for the max-cut problem. For the bipartition problem, in addition we use our own implementation of the \emph{random partition generation algorithm}~\cite[Algorithm~4.6.1]{RubinsteinK04} (since no code for this algorithm is given in~\cite{RubinsteinK04}). This algorithm is used to sample a partition with prescribed sizes of the partition classes and given marginal distributions, which is the only difference between the PBIL for the max-cut and the partition problem.

We note that the Matlab code in~\cite[Pages~274-276]{RubinsteinK04} does not use the artificial margins. We speculate that their absence did not create problems because of the relatively large  population sizes used in their experiments. As discussed in Section~\ref{subsec:cga}, the artificial margins $\{1/n,1-1/n\}$ are usually utilized to avoid the frequencies reaching the absorbing boundaries $0$ or $1$. To obtain a complete picture, we shall conduct experiments both with and without margins. 

We use the following settings, again taken from~\citep{RubinsteinK04}.
\begin{itemize}
\item \textbf{Max-cut problem:} We use a problem size of $n=400$. For the PBIL, the learning rate is $\rho=1$, the selection pressure is $\eta=0.1$, and the sample size is $\lambda=1{,}000$. These are the settings from~\cite[Table~2.3]{RubinsteinK04}, where we note that the notation there writes $(\alpha,\rho,N)$ where we use $(\rho,\eta,\lambda)$.  
\item \textbf{Bipartition problem:} We use a problem size of $n=600$. For the algorithm parameters, we use $\rho=0.7$, $\eta=0.01$, and $\lambda=6{,}000$ as in~\cite[Table~4.7]{RubinsteinK04}.\footnote{\cite[Table~4.7]{RubinsteinK04} does not specify the value of $\rho$, so we used $\rho=0.7$ as specified in the preceding table~\cite[Table~4.6]{RubinsteinK04}.}
\end{itemize}

In the experiments with the smart-restart PBIL, we used the PBIL kernel as above (with the same values for $\rho$ and $\eta$), and only the value of $\lambda$ was set via the smart-restart mechanism. The following shows the hyperparameter choices for the smart-restart PBIL, which were identical for both optimization problems.
\begin{itemize}
\item \textbf{Update factor $U$}: We used the same (natural) value $U = 2$ as in our experiments with the smart-restart cGA. 
\item \textbf{Budget factor $b$}: We used the two factors $b = 96\frac{\eta}{\rho^2}$ and $b = 6\frac{\eta}{\rho^2\ln n}$ (depending on the values of $\eta$ and $\rho$ as used in the PBIL kernel). That is, $b=96\cdot\frac{0.1}{1^2}=9.6$ and $6\frac{0.1}{1^2\ln n}=0.6/\ln n$ as $\eta=0.1$ and $\rho=1$ for PBIL on the max-cut problem~\cite[Table~2.3]{RubinsteinK04}, and $b=96\frac{0.01}{0.7^2}=\frac{96}{49}$ and $6\frac{0.01}{0.7^2\ln n}=\frac{6}{49\ln n}$ as $\eta=0.01$ and $\rho=0.7$ for the bipartition problem~\cite[Table~4.7]{RubinsteinK04}. The motivation for these choices is as follows. We recall that in our experiments on the smart-restart cGA, we chose the budget factor $b=16$ so that the probability of detecting the genetic drift is at least $1/2$. We chose $b=1/\ln n$ since the order $\Theta(1/\log n)$ allows a union bound over all $n$ frequencies, and we chose the precise value $b = 1/\ln(n)$ based on preliminary experiments. For smart-restart PBIL, to ensure a detection probability of at least $1/2$, we set $b=\frac{32\eta}{(2-1/c)\rho^2}$ as discussed in Section~\ref{ssec:srPBIL}. For the $c\in(1/2,1/\sqrt2)$, we chose $c=3/5$ as roughly the middle point in this interval. This explains our choice $b=\frac{32\eta}{(2-5/3)\rho^2}=96\frac{\eta}{\rho^2}$. Since this first value of $b$ is $6\frac{\eta}{\rho^2}$ times the $16$, the first value for smart-restart cGA, for reasons of comparability we set the second value of $b$ also $6\frac{\eta}{\rho^2}$ times $1/\ln n$ (the second value of $b$ for the smart-restart cGA), that is, $6\frac{\eta}{\rho^2\ln n}$. 
\end{itemize}
All algorithms were terminated only when the optimum is found. 

In the experiments with the two restart mechanisms discussed in Section~\ref{ssec:other}, we used the same PBIL kernel with the same $(\eta,\rho)$ and initial  $\lambda=2$ as in the experiments with the smart-restart PBIL. For the AH restart strategy, the memory size $L=10+\lceil 30n/\lambda \rceil$ and the threshold $\epsilon=10^{-12}$ are set as in the original paper~\citep{AugerH05} (their notation for $\epsilon$ is \texttt{Tolfun}). Since in our experiments the runs with AH (when using frequency margins) took extremely long, in our experiments we terminate the algorithm the first time the number of fitness evaluations exceeds $1.5\times 10^7$ for the max-cut problem, and $3\times 10^6$ for the bipartition problem. We note that in our experiments, all 20 runs of our smart-restart PBIL (with or without margins) reached the optimum within these limits.

\subsubsection{Experimental Results and Analyses}
Figure~\ref{fig:pbilnew} shows the runtimes (measured by the number of fitness evaluations) of the original PBIL with margins, our smart-restart PBIL with margins, and HL with margins, utilizing the settings described above. We did not plot AH with margins as it fails to reach the global optimum in all 20 runs. We easily see a good performance of our smart-restart PBIL. For example, in terms of the median runtime (the red line in each box), the smart-restart PBIL with both values for $b$ has a smaller runtime than the PBIL with the population size suggested in~\citep{RubinsteinK04} for both problems, and better than the HL for the max-cut problem. For the bipartition problem, the smart-restart PBIL with smaller $b$ shows a clear superiority over the HL.
We notice the large variances for $b=96\frac{\eta}{\rho^2}$ for the bipartition problem and $b=6\frac{\eta}{\rho^2\ln n}$ for the max-cut problem, but we have no explanation for these. These variances, however, do not change our general impression that the smart-restart approach is generally preferable over using fixed parameters with the values suggested in~\citep{RubinsteinK04}. 

\begin{figure}[!ht]
\centering
\includegraphics[width=5.15in]{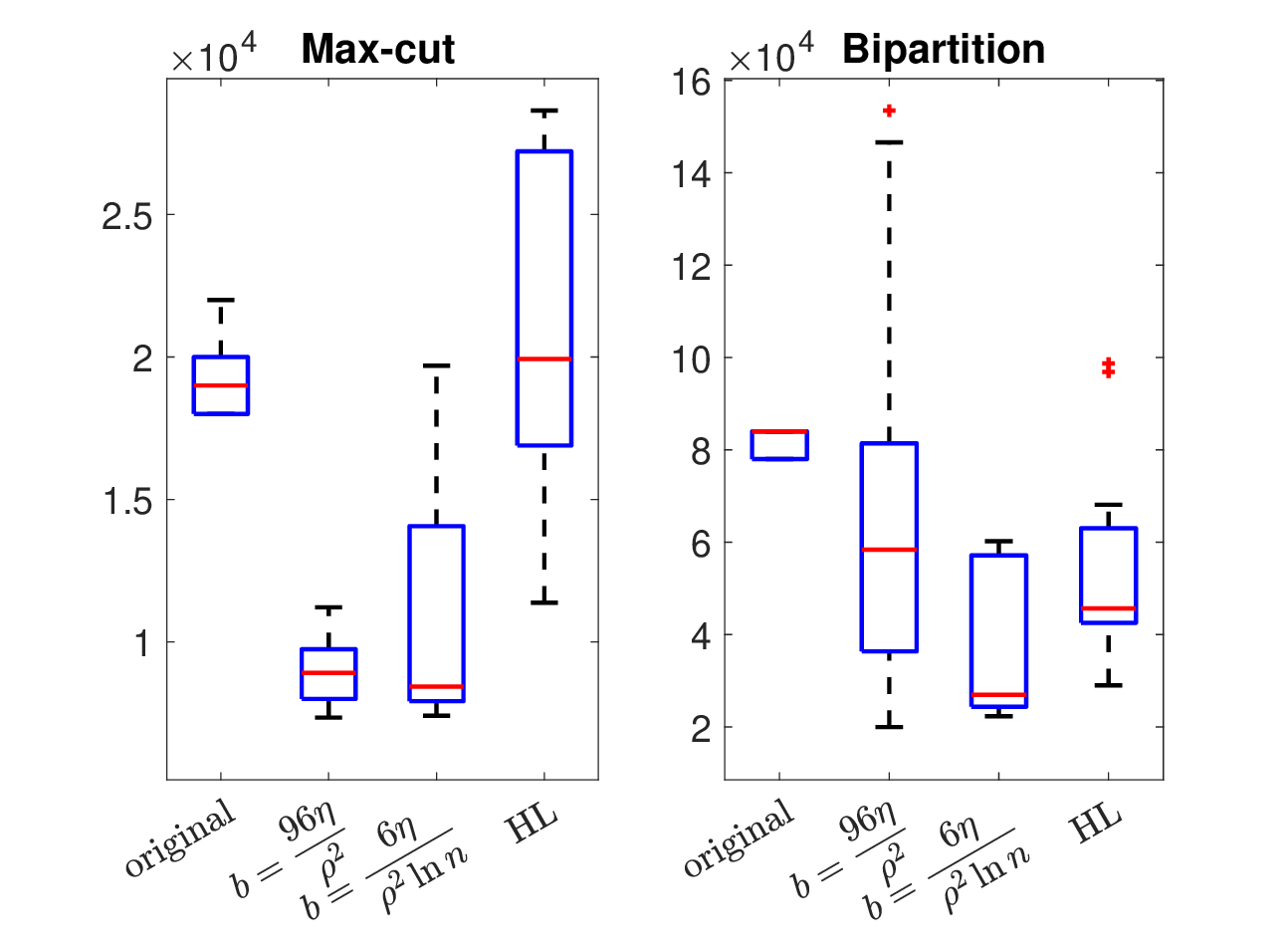}
\caption{Boxplots of runtimes of the original PBIL (with margins), the smart-restart PBIL (with margins) with two budget factors ($b=96\frac{\eta}{\rho^2}$ and $6\frac{\eta}{\rho^2\ln n}$), and HL-restart PBIL (with margins) on the max-cut problem ($n=400$)~\cite[Table~2.3]{RubinsteinK04} and bipartition problem ($n=600$)~\cite[Table~4.7]{RubinsteinK04} in 20 independent runs. }
\label{fig:pbilnew}
\end{figure} 

To understand the influence of using frequency margins, we collect the median results among 20 independent runs of the different algorithms without margins in Table~\ref{tbl:pbilnew}. For comparison, we also record the results from the runs with margins from Figure~\ref{fig:pbilnew}. 
From Table~\ref{tbl:pbilnew}, we deduce that the (common) usage of the margins is usually beneficial also for this problem setting, with the exception of the AH strategy. For  AH, a restart is triggered when for a certain time interval no or only small fitness changes are observed. Here, apparently, margins are not helpful as they prevent the samples from becoming too similar. We note that also without margins, our smart-restart PBIL succeeds in finding the optimum, albeit with larger runtimes. The runtime increase is most drastic for the larger budget factor as here more time is spent in a run, hence more time is lost, when a frequency due to genetic drift has reached the wrong boundary value.

When not using frequency margins, the ranking of the four algorithms is the same for both combinatorial problems: The smart-restart strategy with small $b$ performs  best, then HL, AH, and the smart-restart strategy with the large $b$. This is somewhat intuitive when recalling the ideas behind these strategies. The smart-restart strategy with small $b$ tries to estimate the first time when some frequency reaches a boundary due to genetic drift. HL restarts when all frequencies have actually reached the boundaries. AH restarts when only small fitness changes happen in fixed-length intervals of iterations (where we recall that the original AH strategy contained four more criteria that were specific to the CMA-ES). The smart-restart strategy with large $b$ tries to estimate the time when a fixed frequency (hence also the typical frequency) reaches the boundaries. In this light, it is natural that the smart-restart strategy with a small budget factor performs better than the one with a large factor as it suffices that a single frequency is stuck at the wrong boundary to prevent finding the optimum. The fact that the HL strategy of checking whether all frequencies have reached a boundary performs relatively well, could suggest that the estimates of the smart-restart strategy with a small budget are still relatively conservative, and that possibly a restart could have been triggered even earlier. It is not totally surprising that the general fitness-dependent AH strategy has a harder stand than the strategies based on the mechanics of EDAs  (but, as said, the original AH strategy was designed for the CMA-ES and included criteria exploiting the mechanics of the CMA-ES).

\begin{table}[!ht]
    \centering
    \caption{The median runtime for original, smart-restart, HL-restart, and AH-restart PBIL with and without margins on the max-cut and bipartition problems in 20 independent runs. The minimal median runtimes are in the bold font. 5 of 20 runs for the original PBIL without margins on the max-cut failed to find the optimum due to reaching the wrong boundaries. All 20 runs of the AH-restart PBIL with margins cannot reach the global optimum within the maximal number of the function evaluations.}
    \label{tbl:pbilnew}
    \begin{tabular}{ccccccc}
    \toprule
      Problem & Margins & PBIL & $b=96\frac{\eta}{\rho^2}$ & $b=6\frac{\eta}{\rho^2\ln n}$ & HL & AH \\ \hline
      Max-cut &  With & 19,000 & 8,908 & \bf{8,430} & 19,932 & $\ge1.5\times 10^7$\\ 
       & Without & 20,000 & 849,292 & \bf{19,182} & 20,268 & 119,900\\ \hline
       Bipartition & With & 84,000 & 58,376 & \bf{26,942} & 45,665 & $\ge3\times 10^6$\\ 
        & Without & 84,000 & 702,344 & \bf{26,430} & 61,838 & 221,161\\
        \bottomrule
    \end{tabular}
\end{table}

Similar to the results in Section~\ref{sec:exper}, we observe that our smart-restart PBIL with the smaller budget clearly achieves very good runtimes, no matter with or without margins.

\section{Conclusion}
\label{sec:con}
Choosing the parameters that control the genetic drift of estimation-of-distribution algorithms is one of the key difficulties for the practical usage of EDAs. To overcome this difficulty, we proposed a smart-restart mechanism that removes this parameter from the algorithm. Our mechanism is a simple restart strategy with exponentially growing population size, but different from previous works it sets a prior fitness evaluation budget for each population size based on a recent quantitative analysis estimating when genetic drift is likely to occur. 

Under a reasonable assumption on how the runtime depends on the population size, we theoretically analyzed our scheme and observed that it can lead to asymptotically optimal runtimes for the cGA and the UMDA. 

Via extensive experiments on \onemax, \LO, \jump, and \DLB, we showed the efficiency of the smart-restart cGA, also when compared with the parallel-run cGA. The results for the original cGA with different population sizes experimentally show that the population size is crucial for the performance of the cGA and that the theoretically suggested population size can be far away from the right one.

We also applied our smart-restart mechanism to the PBIL on two combinatorial problems. Our experiments showed again the efficiency of the smart-restart PBIL compared to the PBIL with the original settings and other restart strategies.

The problem of how to cope with genetic drift, naturally, is equally interesting for multivariate EDAs such as~\citep{BonetIV96,PelikanM99,MuhlenbeinM99,HarikLS06,ProbstR20}. For these, however, our theoretical understanding is limited to very few results such as~\citep{ZhangM04,LehreN19foga,DoerrK23tcs}. In particular, a quantitative understanding of genetic drift comparable to~\citep{DoerrZ20tec} is completely missing. Another interesting question is if dynamic choices of the population size in EDAs can be fruitful. In classic EAs, dynamic parameter choices have recently been used very successfully to overcome the difficulty of finding a suitable static parameter value, see, e.g., the survey~\citep{DoerrD20bookchapter}. How to use such ideas for EDAs is currently not at all clear. 

\section*{Acknowledgments}
This work was supported by National Natural Science Foundation of China (Grant No. 62306086), Science, Technology and Innovation Commission of Shenzhen Municipality (Grant No. GXWD20220818191018001), Guangdong Basic and Applied Basic Research Foundation (Grant No. 2019A1515110177).

This work was also supported by a public grant as part of the Investissement d'avenir project, reference ANR-11-LABX-0056-LMH, LabEx LMH.

This work has profited from many scientific discussions at the Dagstuhl Seminar 22081 ``Theory of Randomized Optimization Heuristics'' and the Dagstuhl Seminar 22182 ``Estimation-of-Distribution Algorithms: Theory and Applications''.
}

\newcommand{\etalchar}[1]{$^{#1}$}

\end{document}